\documentclass[11pt]{article}

\usepackage[preprint]{acl}

\usepackage{times}
\usepackage{latexsym}
\usepackage[T1]{fontenc}
\usepackage[utf8]{inputenc}
\usepackage{microtype}
\usepackage{inconsolata}
\usepackage{graphicx}
\usepackage{booktabs}
\usepackage{amsmath,amssymb,amsfonts,amsthm}
\usepackage{bm}
\usepackage{multirow}
\usepackage{xcolor}
\usepackage{colortbl}
\usepackage{array}
\usepackage{tabularx}
\usepackage{enumitem}
\usepackage[ruled,vlined]{algorithm2e}
\SetKwInput{KwInput}{Input}
\SetKwInput{KwOutput}{Output}

\newtheorem{proposition}{Proposition}
\newtheorem{theorem}{Theorem}
\newtheorem{lemma}{Lemma}

\newtheorem{remark}{Remark}

\definecolor{oursbg}{RGB}{253,238,240}
\definecolor{simred}{RGB}{165,30,45}

\newcommand{\method}{\textsc{Vane}}
\newcommand{\E}{\mathbb{E}}
\newcommand{\R}{\mathbb{R}}
\newcommand{\dout}{d'}
\newcommand{\sg}{\mathrm{sg}}

\newcommand{\ip}[2]{\langle #1,#2\rangle}

\newcommand{\csVaneThree}{81.9}

\newcommand{\csMarginThree}{1.1}
\newcommand{\csGainLoRAThree}{3.0}

\newcommand{\csVaneEight}{86.9}

\newcommand{\csBestNameEight}{LD-MoLE}
\newcommand{\csMarginEight}{0.9}
\newcommand{\csGainLoRAEight}{2.5}

\newcommand{\mtMarginThree}{1.5}
\newcommand{\mtGainLoRAThree}{3.9}

\newcommand{\mtMarginEight}{1.3}
\newcommand{\mtGainLoRAEight}{3.5}

\newcommand{\ablTokenRouter}{1.4}
\newcommand{\ablProduct}{2.0}
\newcommand{\ablLatent}{1.3}
\newcommand{\ablNorm}{1.9}
\newcommand{\ablRaw}{0.6}
\newcommand{\ablUntied}{0.5}
\newcommand{\ablNoDense}{0.9}
\newcommand{\ablIntoExperts}{0.3}
\newcommand{\ablPairBal}{0.5}
\newcommand{\ablNoBal}{1.1}
\newcommand{\ablPaired}{1.0}
\newcommand{\ablShared}{0.8}
\newcommand{\ablSoftmax}{0.7}
\newcommand{\ablSeparable}{1.7}
\newcommand{\ablCompassOut}{0.2}

\newcommand{\rhoNoDense}{0.43}

\newcommand{\domGainMath}{1.8}
\newcommand{\domGainCode}{2.0}
\newcommand{\domGainMed}{0.5}
\newcommand{\domGainCS}{1.1}
\newcommand{\mtMarginQwenTwo}{1.3}
\newcommand{\csMarginQwenTwo}{0.9}
\newcommand{\mtMarginQwenThree}{1.2}
\newcommand{\csMarginQwenThree}{0.8}
\newcommand{\timeVaneOverLoRA}{1.26}
\newcommand{\timeVaneVsLDMoLE}{34}
\newcommand{\latVane}{27.2}
\newcommand{\latLDMoLE}{36.2}

\newcommand{\lowdataMarginTen}{1.6}
\newcommand{\rhoVane}{0.61}
\newcommand{\rhoTok}{0.22}
\newcommand{\rhoLat}{0.24}
\newcommand{\rhoProd}{0.12}
\newcommand{\rhoNorm}{0.06}
\newcommand{\gainVane}{0.78}
\newcommand{\gainTok}{0.41}
\newcommand{\gainNorm}{0.04}
\newcommand{\cosMLP}{0.29}
\newcommand{\cosAttn}{0.21}
\newcommand{\pVane}{1.06}

\newcommand{\pMoE}{2.21}
\newcommand{\pLDMoLE}{2.21}
\newcommand{\scaleGainOneB}{4.4}

\newcommand{\miReader}{0.02}
\newcommand{\miWriter}{0.60}

\title{Score the Update, Not the Token:\\Descent-Aligned Routing for Combinatorial LoRA Experts}

\author{
  \textbf{Priya Nair\textsuperscript{1}},
  \textbf{Lukas Brenner\textsuperscript{2}},
  \textbf{Maya Lindqvist\textsuperscript{1}},
  \textbf{Daniel Whitmore\textsuperscript{2}},
  \\
  \textbf{Wen-Hsuan Liu\textsuperscript{1}},
  \textbf{Tom Saliencro\textsuperscript{1}},
  \textbf{Amara Okonkwo\textsuperscript{2}},
  \textbf{Rohan Desai\textsuperscript{2}}
  \\
  \\
  \textsuperscript{1}University of California, Irvine \quad
  \textsuperscript{2}University of Washington \\
  \texttt{saliencro@gmail.com}
}

\begin{document}
\maketitle

\begin{abstract}
Mixture-of-LoRA-experts methods raise the capacity of low-rank adaptation by routing each token to a few low-rank experts. Nearly all of them tie one input-side factor to one output-side factor per expert, and nearly all of them route by scoring the token: the router picks experts without seeing what any of them would write. We argue that the router should score the update. To first order, adding an expert's update to a layer output lowers the loss by the inner product between that update and the negative loss gradient at the output. This usefulness is quadratic in the token, so a router that is linear in the token sees only the part of it that runs through the token mean, and routers that rank experts by the norm of their own activations never see the output factor. If each expert is split into a reader (down-projection) and a writer (up-projection), the usefulness of every reader--writer pair becomes an inner product in the shared rank-$r$ space, and all $N_AN_B$ pairs can be scored from $N_A+N_B$ vectors. We build \method{} on this identity. A low-rank compass predicts the descent direction of each token. \method{} scores every pair by the alignment between its update and the compass without forming any update, activates the top-$k$ pairs with additive gates, and gives every pair its exact first-order router gradient. On single-domain commonsense reasoning and a four-domain multi-task mixture with Llama-3.2-3B and Llama-3.1-8B, \method{} attains the best average among twelve PEFT and MoE-LoRA baselines, by \csMarginEight--\csMarginThree{} and \mtMarginEight--\mtMarginThree{} points respectively, with less than half the trainable parameters of an 8-expert MoE-LoRA. Its router scores also track the measured usefulness of experts far more closely than token routers do.
\end{abstract}

\section{Introduction}
\label{sec:intro}

\begin{figure*}[t]
\centering
\includegraphics[width=\textwidth]{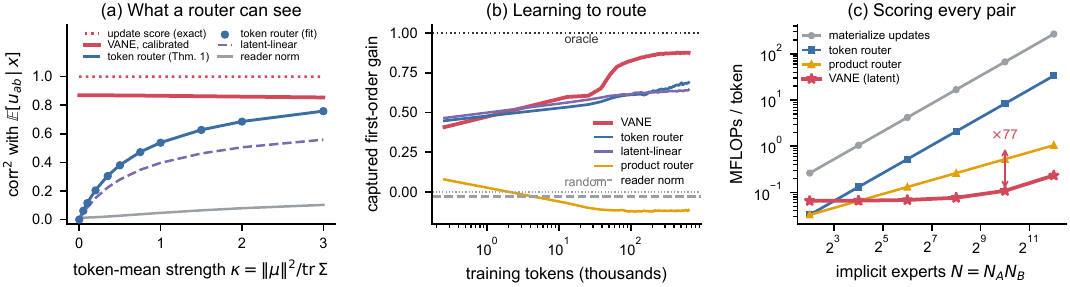}
\caption{\textbf{What a router can see} (exact computations in a controlled linear-response model with 64 reader--writer pairs; Appendix~\ref{app:toy}). (a) Squared correlation between router scores and the expected usefulness $\E[u_{ab}\mid x]$ as the token mean grows. Token routers follow the ceiling of Theorem~\ref{thm:router} (line: closed form; dots: least-squares fit) and see nothing for centered tokens. Reader norms are writer-blind. The update score is exact, and \method{}'s calibrated version stays at $0.87$ for every $\kappa$. (b) Routers trained online with the same dense first-order signal: held-out fraction of the attainable first-order gain captured by their top-4 pairs. (c) FLOPs per token to score every pair at width 4096. \method{} scores 1024 pairs $77\times$ more cheaply than a linear router.}
\label{fig:motivation}
\end{figure*}

Low-rank adaptation (LoRA) fine-tunes a frozen language model by adding a trainable product $BA$ to each weight matrix~\citep{hu2022lora,houlsby2019parameter}. When the training data mixes domains whose updates conflict, the standard remedy is a mixture of low-rank experts: a router sends each token to a few experts, so capacity grows while per-token cost stays small~\citep{shazeer2017outrageously,fedus2022switch,dou2024loramoe,li2024mixlora,wu2024mole,tian2024hydralora,fan2025goat}. Follow-up work refines how experts are allocated across layers~\citep{gao2024mola,liao2025hmora}, how many are active~\citep{zhuang2026ldmole,li2025dynmole,saliencro2026care}, how finely capacity is split~\citep{zhao2025smora,zou2025flylora,zeng2025smore}, and how experts share parameters~\citep{wang2025malora,cao2026comol}.

These designs share two conventions. Each expert welds one input-side factor $A_i$ to one output-side factor $B_i$, and the router \emph{scores the token}: it maps the hidden state $x$ to logits without seeing what any expert would compute. Combinatorial designs mix and match factors~\citep{zhang2024trex,zhao2025smora,zeng2025smore}, and self-routing lets experts select themselves by the size of their own activations~\citep{lv2025aoe,zou2025flylora,han2025furina,hou2026passmoe}. But no design states what a router should estimate, so routers are refined by heuristics such as entropy control, dynamic $k$, expert communication, and semantic priors~\citep{li2025dynmole,zhuang2026ldmole,mu2026talklora,shi2026samora}.

We start from that question. Adding an expert's update $v$ to a layer output $y$ changes the loss, to first order, by $-\ip{\delta}{v}$, where $\delta=-\nabla_y\ell$ is the descent direction at the output. The \emph{usefulness} of an expert for a token is thus the alignment between the update it would write and the direction in which the loss wants the output to move. Two consequences follow (Figure~\ref{fig:motivation}). \emph{Usefulness is second order in the token.} Under a linear-response model, a router that is linear in $x$ sees only the part of usefulness that runs through the token mean, and nothing for centered tokens (Theorem~\ref{thm:router}). Norm-based self-routing sees only the input factor, so it cannot tell apart experts that read the same features but write different things. \emph{Usefulness factorizes once experts are split.} If experts are built from $N_A$ \emph{readers} (down-projections) and $N_B$ \emph{writers} (up-projections), the usefulness of pair $(a,b)$ is $\ip{B_b^{\top}\delta}{A_ax}$, an inner product in the rank-$r$ latent space, so all $N_AN_B$ pairs can be scored from $N_A+N_B$ latent vectors.

We build \method{} on these two facts; its router turns toward the descent direction like a weather vane toward the wind. A \method{} layer keeps readers and writers of rank $r$ and a low-rank \emph{compass} $e(x)$ that predicts $\delta$. Each writer turns the compass into a latent query, and every pair is scored by the alignment between its update and the compass, per unit of update norm, using $r\times r$ writer Gram matrices so that no update is formed. The top-$k$ pairs are combined with additive gates. In the backward pass, the same latent algebra yields the exact first-order usefulness of \emph{every} pair; a one-line straight-through core turns it into a dense router gradient, and each step on the compass's output map provably turns it toward the descent direction. Our contributions:
\begin{itemize}[leftmargin=*,itemsep=1pt,topsep=2pt]
\item \textbf{Analysis.} We identify first-order usefulness as the target of routing and show that token routers are capped by a closed-form ceiling, reader-side scores are writer-blind, and update scores are exact (\S\ref{sec:problem}).
\item \textbf{Method.} \method{} combines reader--writer experts, latent update scores, additive gates, and an exact dense router gradient; its routing is invariant to how experts are parameterized (\S\ref{sec:method}).
\item \textbf{Evidence.} On two Llama backbones, \method{} beats twelve baselines in single-domain and multi-task adaptation with half the parameters of 8-expert mixtures, and its scores track measured usefulness ($\rho=\rhoVane$ vs.\ $\le\rhoLat$ for token routers; \S\ref{sec:experiments}).
\end{itemize}

\section{Related Work}
\label{sec:related}

\paragraph{Mixtures of LoRA experts.}
LoRAMoE, MixLoRA, and MoLE route tokens among paired LoRA experts~\citep{dou2024loramoe,li2024mixlora,wu2024mole}. Later work allocates experts by depth~\citep{gao2024mola,liao2025hmora}, aligns their optimization with full fine-tuning~\citep{fan2025goat}, makes $k$ dynamic~\citep{li2025dynmole,zhuang2026ldmole,saliencro2026care,saliencro2026vimole}, lets experts communicate~\citep{mu2026talklora}, injects semantic priors~\citep{shi2026samora}, or uses heterogeneous and spectral experts~\citep{cao2026moa,jiang2026fouriermoe,saliencro2026frame}; others drop the router~\citep{zha2026cdlora,liu2025alignlora}, and clean routing alone does not prevent interference inside an adapter~\citep{chowdhury2026routing}. These methods improve \emph{how} the router decides; we specify \emph{what} it should estimate.

\paragraph{Shared and combinatorial factors.}
LoRA's $A$ extracts features and $B$ produces the output~\citep{zhu2024asymmetry}. HydraLoRA shares one $A$ across $B$ heads~\citep{tian2024hydralora}, MALoRA shares a down-projection subspace~\citep{wang2025malora}, MoSLoRA mixes the two subspaces~\citep{wu2024moslora}, and CoMoL keeps experts as core matrices between shared factors~\citep{cao2026comol}. Mix-and-match designs build many experts from few pieces: rank-one experts~\citep{zhang2024trex,zhao2025smora}, residual trees~\citep{zeng2025smore}, shards and vector banks~\citep{wang2025mos,li2024vblora}, and factorized or product-key expert tensors~\citep{oldfield2024mumoe,lample2019product,he2024peer}. \method{} uses the simplest such structure, readers times writers, and is the first to exploit it for exact scoring.

\paragraph{Coupling routing and computation.}
Autonomy-of-Experts ranks experts by the norm of their own low-rank activations~\citep{lv2025aoe}. FlyLoRA routes with a frozen sparse projection~\citep{zou2025flylora}, FURINA by the angle between input and adapter~\citep{han2025furina}, Self-Routing by a subspace of the hidden state~\citep{mohamud2026selfrouting}, and CoMoL and PASs-MoE in the adapters' latent space~\citep{cao2026comol,hou2026passmoe}; all read the input side only. Default MoE densifies the router gradient with moving averages of expert outputs~\citep{panda2025dense}, and GRIN estimates gradients through discrete routing~\citep{liu2024grin}. \method{}'s scores involve the writers and a forecast of the loss gradient, and its dense gradient is exact. Appendix~\ref{app:related} extends this review.

\section{What Should a Router Compute?}
\label{sec:problem}

\paragraph{Setup.}
Let $W_0\in\R^{\dout\times d}$ be a frozen linear layer with input $x\in\R^d$ and output $y$. A \emph{combinatorial} low-rank adapter keeps $N_A$ readers $A_a\in\R^{r\times d}$ and $N_B$ writers $B_b\in\R^{\dout\times r}$. Each pair $(a,b)$ is an implicit expert with update $v_{ab}(x)=B_bz_a$, where $z_a=A_ax$ is the reader's latent code. With gates $g_{ab}(x)\ge0$ on a set $\mathcal{S}(x)$ of $k$ active pairs, the layer computes
\begin{equation}
y=W_0x+\tfrac{\alpha}{r}\textstyle\sum_{(a,b)\in\mathcal{S}(x)}g_{ab}(x)\,B_bA_ax.
\label{eq:layer}
\end{equation}
Paired MoE-LoRA is the special case $N_A=N_B$ with only diagonal pairs allowed, and HydraLoRA the case $N_A=1$ (Appendix~\ref{app:unify}). Let $\delta(x)=-\nabla_y\ell$ be the descent direction of the token's loss $\ell$ at this layer's output.

\paragraph{Usefulness.}
Adding $\eta\,v_{ab}$ to $y$ changes the loss by $-\eta\,u_{ab}+O(\eta^2)$, where
\begin{equation}
u_{ab}(x)=\ip{\delta(x)}{B_bA_ax}=\ip{B_b^{\top}\delta(x)}{z_a}.
\label{eq:useful}
\end{equation}
We call $u_{ab}$ the \emph{usefulness} of pair $(a,b)$ for the token. Under additive gates, the first-order gain of a selection is $\sum_{\mathcal{S}}g_{ab}u_{ab}$, so a router should rank pairs by their usefulness. It must do so from $x$ alone, because $\delta$ depends on the label.

\begin{proposition}[Usefulness lives in the latent space]
\label{prop:latent}
For every token, the $N_A\times N_B$ usefulness matrix is $Z^{\top}W$ with $Z=[z_1,\dots,z_{N_A}]$ and $W=[B_1^{\top}\delta,\dots,B_{N_B}^{\top}\delta]$. The squared update norms are $\|v_{ab}\|^2=z_a^{\top}G_bz_a$ with $G_b=B_b^{\top}B_b$. All of these cost $O(r(N_Ad+N_B\dout)+N_AN_Br^2)$ operations, whereas materializing the updates costs $O(N_AN_Br\dout)$.
\end{proposition}

\paragraph{Which routers can see usefulness?}
To say what a router must represent, we model the predictable part of the descent direction as linear in the token.

\begin{theorem}[Router expressivity]
\label{thm:router}
Let $\delta(x)=\Gamma x+\xi$ with $\E[\xi\mid x]=0$ and $x\sim\mathcal{N}(\mu,\Sigma)$. Then $\bar u_{ab}(x):=\E[u_{ab}\mid x]=x^{\top}M_{ab}x$ with $M_{ab}=\Gamma^{\top}B_bA_a$. Let $S=\tfrac12(M_{ab}+M_{ab}^{\top})$.
\begin{enumerate}[leftmargin=*,itemsep=0pt,topsep=1pt,label=(\roman*)]
\item \emph{Token routers.} Every score affine in $x$ satisfies
\begin{equation*}
\mathrm{corr}^2(s,\bar u_{ab})\le\frac{4\mu^{\top}S\Sigma S\mu}{4\mu^{\top}S\Sigma S\mu+2\mathrm{tr}((S\Sigma)^2)},
\end{equation*}
with equality for the least-squares router. This class includes linear routers over pairs, product routers, and linear routers applied to reader latents. The ceiling is $0$ when $\mu=0$.
\item \emph{Reader-side scores.} A score that depends on the token only through $A_ax$ gives the same value to every writer of reader $a$. Within a reader, it therefore captures only the gain of a uniformly random writer.
\item \emph{Update scores.} The score $\ip{e(x)}{v_{ab}(x)}$ with $e(x)=UV^{\top}x$ equals $\bar u_{ab}$ for every pair and token if $UV^{\top}=\Gamma$. This is attainable whenever the rank $r_e\ge\mathrm{rank}\,\Gamma$, and then the top-$k$ pairs under the score maximize the expected first-order gain among all selections based on $x$.
\end{enumerate}
\end{theorem}

Token routers thus work to the extent that the token mean interacts with usefulness. Pretrained representations are strongly anisotropic~\citep{ethayarajh2019contextual}, which is why token routers work at all, but the token-specific part of usefulness stays invisible to them; Figure~\ref{fig:motivation}a matches the closed form to within $0.002$. The remedy in (iii) is not a larger router but a different object: a vector in output space whose inner product with each candidate update predicts its effect on the loss. Proofs are in Appendix~\ref{app:proofs}, and Appendix~\ref{app:numerics} checks them numerically.

\section{\method{}}
\label{sec:method}

\begin{figure*}[t]
\centering
\includegraphics[width=\textwidth]{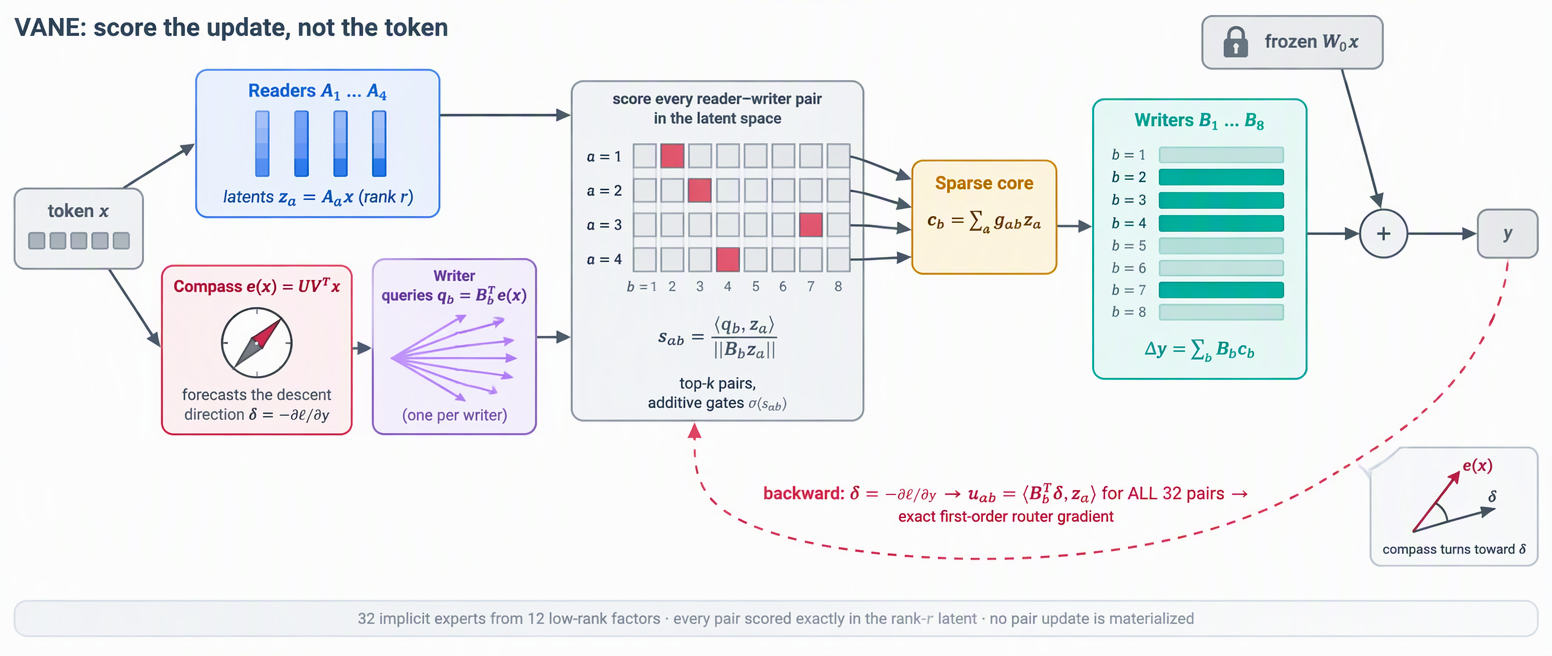}
\caption{\textbf{\method{} for one token.} Readers produce rank-$r$ latents. A low-rank compass forecasts the descent direction at the layer output, and each writer turns it into a query in the same latent space. Every reader--writer pair is scored by the alignment between its update and the compass without materializing the update. The top-$k$ pairs are summed with additive gates. In the backward pass, the exact first-order usefulness of every pair comes from the same latents and gives the router a dense gradient.}
\label{fig:overview}
\end{figure*}

\method{} instantiates Theorem~\ref{thm:router}(iii) inside the adapter of Eq.~\eqref{eq:layer} (Figure~\ref{fig:overview}, Algorithm~\ref{alg:vane}).

\paragraph{Compass and writer queries.}
The compass $e(x)=UV^{\top}x$, with $U\in\R^{\dout\times r_e}$ and $V\in\R^{d\times r_e}$, forecasts the descent direction. It is never added to the output; it only scores updates. Each writer turns it into a latent query
\begin{equation}
q_b(x)=B_b^{\top}e(x)=K_b\,V^{\top}x,
\label{eq:query}
\end{equation}
where the $r\times r_e$ key $K_b=B_b^{\top}U$ is computed once per forward pass. Queries move with their writers, so the router never lags behind the experts.

\paragraph{Calibrated update scores.}
The alignment between the compass and a pair's update is $\ip{e}{v_{ab}}=\ip{q_b}{z_a}$ by Eq.~\eqref{eq:useful}. \method{} scores alignment per unit of update norm,
\begin{equation}
s_{ab}(x)=\frac{\ip{q_b(x)}{z_a(x)}}{\sqrt{z_a^{\top}G_bz_a}+\epsilon},
\label{eq:score}
\end{equation}
that is, the length of the compass projected onto the direction of $v_{ab}$. Calibration makes routing depend only on update directions (Proposition~\ref{prop:invariance}). Without it, writers that grow are selected more often, receive more gradient, and grow further. Magnitude still reaches the output through $v_{ab}$. Scores use stop-gradients on $z_a$ and $B_b$, so the compass is the only router parameter.

\paragraph{Sparse additive core.}
The active set $\mathcal{S}(x)$ holds the $k$ highest-scoring pairs, with gates $g_{ab}=\sigma(s_{ab})$. Each writer receives the gated latents $c_b=\sum_{a:(a,b)\in\mathcal{S}}g_{ab}z_a$, and
\begin{equation}
y=W_0x+\tfrac{\alpha}{r}\textstyle\sum_bB_bc_b,
\label{eq:core}
\end{equation}
so the layer is two dense low-rank products around an $N_A\times N_B$ sparse core. Rank-wise mixtures are the special case of a diagonal core over rank-one factors~\citep{zhao2025smora,zou2025flylora}. The gates are additive sigmoids, not softmax-normalized: first-order gains add up, and complementary pairs should stack rather than be averaged (Proposition~\ref{prop:rank}; cf.\ \citealp{nguyen2025sigmoid}).

\paragraph{Exact dense router gradient.}
A top-$k$ router normally learns only from the pairs it selected. \method{} replaces the gate of each unselected pair in $c_b$ by $g_{ab}-\sg(g_{ab})$, which is zero in the forward pass but has the derivative of an active gate. Every pair then receives the gradient $-\tfrac{\alpha}{r}\sigma'(s_{ab})\,u_{ab}$, its exact first-order effect had it been active, at $O(N_AN_Br)$ extra cost per token and without materializing any expert output. Default MoE approximates it with averaged past outputs~\citep{panda2025dense}.

\paragraph{Balancing.}
Parameters live in readers and writers, so we balance the two marginals. Let $f^A_a,f^B_b$ be the fractions of active pairs that use reader $a$ or writer $b$, and $P^A_a,P^B_b$ the marginals of $\mathrm{softmax}(s)$. We add $\lambda_{\mathrm{bal}}\mathcal{L}_{\mathrm{bal}}$ with $\mathcal{L}_{\mathrm{bal}}=N_A\sum_af^A_aP^A_a+N_B\sum_bf^B_bP^B_b$~\citep{fedus2022switch}, which leaves the joint pair distribution free to specialize. Algorithm~\ref{alg:vane} (Appendix~\ref{app:impl}) summarizes one training step.

\paragraph{Properties.}
Three guarantees follow (Appendix~\ref{app:proofs}).

\begin{proposition}[The compass turns toward the descent]
\label{prop:dense}
With the straight-through core, $\partial\ell/\partial s_{ab}=-\tfrac{\alpha}{r}\sigma'(s_{ab})u_{ab}$ for all pairs. For a single token, a gradient step of size $\eta$ on $U$ changes the compass by $\Delta e=\eta\tfrac{\alpha}{r}\|h\|^2P\delta$ with $P=\sum_{ab}\sigma'(s_{ab})\,v_{ab}v_{ab}^{\top}/(\|v_{ab}\|+\epsilon)\succeq0$. Hence $\ip{\Delta e}{\delta}\ge0$: each such step turns the compass toward $\delta$ within the span of the candidate updates.
\end{proposition}

\begin{proposition}[Routing depends on what experts do]
\label{prop:invariance}
For any $M\in\mathrm{GL}(r)$ and $c_a,c'_b>0$, the map $A_a\mapsto c_aMA_a$, $B_b\mapsto c'_bB_bM^{-1}$ leaves every score in Eq.~\eqref{eq:score} unchanged (as $\epsilon\to0$). With $c\equiv1$ it also leaves the layer output unchanged. Reader-norm scores $\|A_ax\|$ and latent-linear scores $w^{\top}A_ax$ are not invariant under $M$.
\end{proposition}

\begin{proposition}[Pair gates are needed for rank]
\label{prop:rank}
If the stacked readers have full row rank and the stacked writers full column rank, the token update in Eq.~\eqref{eq:layer} has rank $r\cdot\mathrm{rank}(G_{\mathcal{S}})\le r\min(k,N_A,N_B)$, where $G_{\mathcal{S}}$ is the writer-by-reader matrix of active gates. Separable gates $g_{ab}=\alpha_a\beta_b$ collapse this to rank $r$. Additive gates reach the first-order gain $\sum_{\mathcal{S}}g_{ab}u_{ab}$, while softmax-normalized gates are capped at $\max_{\mathcal{S}}u_{ab}$.
\end{proposition}

\paragraph{Cost.}
With $N_A{=}4$, $N_B{=}8$, $r{=}4$, and $r_e{=}8$, a \method{} layer has 32 implicit experts and $r(N_Ad+N_B\dout)+r_e(d+\dout)$ parameters. That is \pVane\% of Llama-3.1-8B, versus \pMoE\% for eight paired experts of rank 8. Its compute is dominated by two dense products of rank 16 and 32. Scoring every pair adds $O(r_ed+N_AN_Br^2)$ per token, and there is no gather or scatter over experts.

\begin{table*}[t]
\centering
\footnotesize
\setlength{\tabcolsep}{4.1pt}
\renewcommand{\arraystretch}{0.93}
\begin{tabular}{l c cccccccc c c}
\toprule
\textbf{Method} & \textbf{Param\,\%} & \textbf{BoolQ} & \textbf{PIQA} & \textbf{SIQA} & \textbf{HellaS.} & \textbf{WinoG.} & \textbf{ARC-e} & \textbf{ARC-c} & \textbf{OBQA} & \textbf{Avg.} & $\bm{\Delta}$ \\
\midrule
\multicolumn{12}{c}{\textit{\textbf{Llama-3.2-3B}}}\\
\midrule
LoRA & 0.76 & 67.9 & 83.2 & 76.8 & 88.6 & 79.4 & 84.7 & 71.3 & 79.6 & 78.9 & -- \\
DoRA & 0.78 & 68.5 & 83.8 & 77.4 & 88.9 & 80.1 & 85.1 & 72.2 & 80.2 & 79.5 & +0.6 \\
MoSLoRA & 0.76 & 68.5 & 83.5 & 77.3 & 88.9 & 80.0 & 85.1 & 72.0 & 80.2 & 79.4 & +0.5 \\
CoMoL & 0.77 & 69.2 & 84.1 & 77.8 & 89.4 & 80.6 & 85.4 & 73.0 & 80.7 & 80.0 & +1.1 \\
FlyLoRA & 0.77 & 69.1 & 84.2 & 78.0 & 89.3 & 80.8 & 85.8 & 73.5 & 81.1 & 80.2 & +1.3 \\
\cmidrule(lr){1-12}
HydraLoRA & 1.92 & 68.9 & 84.0 & 77.8 & 89.2 & 80.7 & 85.6 & 72.8 & 80.6 & 80.0 & +1.1 \\
MixLoRA & 2.13 & 69.3 & 84.2 & 78.1 & 89.2 & 80.8 & 85.6 & 72.9 & 81.0 & 80.1 & +1.2 \\
MoLA & 2.01 & 69.1 & 83.9 & 77.9 & 89.3 & 80.7 & 85.6 & 73.3 & 80.6 & 80.1 & +1.2 \\
HMoRA & 3.40 & 69.7 & 84.5 & 77.8 & 89.6 & 81.2 & 86.0 & 73.5 & 81.2 & 80.4 & +1.5 \\
GOAT & 3.22 & 69.7 & 84.5 & 78.3 & \underline{89.8} & 81.4 & 85.8 & 74.0 & \underline{81.5} & 80.6 & +1.7 \\
SMoRA & 4.52 & 69.3 & 84.4 & 78.2 & 89.7 & 80.7 & 85.9 & 73.4 & 81.1 & 80.3 & +1.4 \\
LD-MoLE & 3.22 & \underline{70.1} & \textbf{85.0} & 78.2 & \textbf{89.9} & \underline{82.0} & 85.9 & \underline{74.1} & 81.4 & \underline{80.8} & +1.9 \\
\cmidrule(lr){1-12}
\rowcolor{oursbg}\method{}-S & 0.76 & 69.7 & 84.8 & \underline{78.6} & 89.2 & 81.2 & \underline{86.1} & \underline{74.1} & \underline{81.5} & 80.7 & +1.8 \\
\rowcolor{oursbg}\method{} & 1.52 & \textbf{71.2} & \underline{84.9} & \textbf{79.5} & \underline{89.8} & \textbf{82.9} & \textbf{87.4} & \textbf{76.2} & \textbf{82.9} & \textbf{81.9} & +3.0 \\
\midrule
\multicolumn{12}{c}{\textit{\textbf{Llama-3.1-8B}}}\\
\midrule
LoRA & 0.52 & 72.6 & 87.9 & 80.1 & 94.1 & 85.3 & 90.4 & 79.6 & 85.2 & 84.4 & -- \\
DoRA & 0.54 & 73.1 & 88.2 & 80.5 & 94.5 & 85.9 & 90.7 & 80.4 & 85.8 & 84.9 & +0.5 \\
MoSLoRA & 0.52 & 73.1 & 88.2 & 80.4 & 94.2 & 85.8 & 90.9 & 80.1 & 85.6 & 84.8 & +0.4 \\
CoMoL & 0.53 & 73.6 & 88.5 & 80.9 & 94.6 & 86.4 & 91.2 & 80.9 & 86.3 & 85.3 & +0.9 \\
FlyLoRA & 0.55 & 73.8 & 88.9 & 80.8 & 94.6 & 86.8 & 91.2 & 81.4 & 86.5 & 85.5 & +1.1 \\
\cmidrule(lr){1-12}
HydraLoRA & 1.34 & 73.6 & 88.6 & 80.9 & 94.6 & 86.2 & 91.0 & 80.8 & 86.0 & 85.2 & +0.8 \\
MixLoRA & 1.57 & 73.7 & 88.7 & 81.1 & 94.7 & 86.3 & 91.2 & 81.1 & 86.5 & 85.4 & +1.0 \\
MoLA & 1.38 & 73.5 & 88.4 & 80.7 & 94.8 & 86.4 & 91.2 & 81.4 & 86.0 & 85.3 & +0.9 \\
HMoRA & 2.34 & 74.1 & 89.0 & 81.3 & 95.0 & 86.9 & 91.5 & 81.3 & 86.6 & 85.7 & +1.3 \\
GOAT & 2.21 & 74.0 & \textbf{89.2} & 81.3 & \textbf{95.4} & 87.2 & 91.6 & 81.8 & 86.8 & 85.9 & +1.5 \\
SMoRA & 3.08 & 73.9 & 88.8 & 81.1 & 94.9 & 86.8 & 91.5 & 81.2 & 86.7 & 85.6 & +1.2 \\
LD-MoLE & 2.21 & \underline{74.6} & \underline{89.1} & \underline{81.6} & 95.0 & \underline{87.3} & 91.6 & \underline{81.9} & \underline{86.9} & \underline{86.0} & +1.6 \\
\cmidrule(lr){1-12}
\rowcolor{oursbg}\method{}-S & 0.53 & 74.4 & 89.0 & 81.2 & 95.0 & 86.7 & \underline{91.8} & 81.5 & 86.8 & 85.8 & +1.4 \\
\rowcolor{oursbg}\method{} & 1.06 & \textbf{75.4} & \underline{89.1} & \textbf{82.4} & \underline{95.3} & \textbf{88.5} & \textbf{92.4} & \textbf{83.9} & \textbf{87.9} & \textbf{86.9} & +2.5 \\
\bottomrule
\end{tabular}

\caption{\textbf{Single-domain commonsense reasoning} (accuracy, \%, mean of three seeds; trained on Commonsense170K). \textbf{Param\,\%}: trainable parameters relative to the backbone. $\Delta$: average gain over LoRA. Groups: single-adapter PEFT, mixtures of LoRA experts, ours. \textbf{Bold}/\underline{underline}: best/second best.}
\label{tab:cs}
\end{table*}

\begin{table*}[t]
\centering
\footnotesize
\setlength{\tabcolsep}{3.8pt}
\renewcommand{\arraystretch}{0.93}
\begin{tabular}{l c cccccccc c c}
\toprule
\textbf{Method} & \textbf{Param\,\%} & \textbf{GSM8K} & \textbf{MATH} & \textbf{HumanE.} & \textbf{MBPP} & \textbf{MedMC.} & \textbf{PubMed.} & \textbf{ARC-c} & \textbf{WinoG.} & \textbf{Avg.} & $\bm{\Delta}$ \\
\midrule
\multicolumn{12}{c}{\textit{\textbf{Llama-3.2-3B}}}\\
\midrule
LoRA & 0.76 & 52.7 & 12.6 & 36.0 & 44.8 & 50.2 & 72.4 & 69.5 & 73.8 & 51.5 & -- \\
DoRA & 0.78 & 53.6 & 13.3 & 37.0 & 45.5 & 50.9 & 72.9 & 70.0 & 74.3 & 52.2 & +0.7 \\
MoSLoRA & 0.76 & 53.6 & 13.3 & 36.8 & 45.5 & 50.7 & 72.8 & 70.0 & 74.1 & 52.1 & +0.6 \\
CoMoL & 0.77 & 54.8 & 14.0 & 38.0 & 46.3 & 51.3 & 73.3 & 71.1 & 75.2 & 53.0 & +1.5 \\
FlyLoRA & 0.77 & 55.2 & 14.7 & \underline{39.5} & 46.8 & 51.9 & 73.8 & 71.1 & 75.1 & 53.5 & +2.0 \\
\cmidrule(lr){1-12}
HydraLoRA & 1.92 & 55.0 & 14.1 & 38.0 & 46.6 & 51.6 & 73.9 & 71.2 & 75.3 & 53.2 & +1.7 \\
MixLoRA & 2.13 & 54.8 & 14.4 & 38.6 & 46.8 & 51.9 & 73.4 & 71.1 & 75.4 & 53.3 & +1.8 \\
MoLA & 2.01 & 54.9 & 14.0 & 38.3 & 46.8 & 51.7 & 73.4 & 71.2 & 74.6 & 53.1 & +1.6 \\
HMoRA & 3.40 & 55.6 & 14.7 & 38.6 & 47.4 & 52.2 & 73.6 & 71.5 & 75.2 & 53.6 & +2.1 \\
GOAT & 3.22 & 55.1 & 14.7 & 39.2 & 47.4 & 52.3 & \underline{74.0} & 71.6 & 75.3 & 53.7 & +2.2 \\
SMoRA & 4.52 & 55.4 & 14.5 & 38.7 & 46.9 & 51.6 & 73.5 & 71.5 & 75.1 & 53.4 & +1.9 \\
LD-MoLE & 3.22 & \underline{55.7} & \underline{15.1} & 39.2 & \underline{47.5} & \underline{52.5} & \textbf{74.1} & 71.2 & 75.8 & \underline{53.9} & +2.4 \\
\cmidrule(lr){1-12}
\rowcolor{oursbg}\method{}-S & 0.76 & \underline{55.7} & 14.5 & 39.3 & 47.3 & 52.4 & 73.7 & \underline{71.8} & \underline{76.5} & \underline{53.9} & +2.4 \\
\rowcolor{oursbg}\method{} & 1.52 & \textbf{57.8} & \textbf{17.1} & \textbf{41.5} & \textbf{49.6} & \textbf{53.0} & 73.9 & \textbf{73.7} & \textbf{76.9} & \textbf{55.4} & +3.9 \\
\midrule
\multicolumn{12}{c}{\textit{\textbf{Llama-3.1-8B}}}\\
\midrule
LoRA & 0.52 & 66.4 & 19.8 & 48.2 & 55.9 & 57.3 & 75.6 & 78.1 & 81.7 & 60.4 & -- \\
DoRA & 0.54 & 67.3 & 20.3 & 49.0 & 56.7 & 57.8 & 75.9 & 78.6 & 82.1 & 61.0 & +0.6 \\
MoSLoRA & 0.52 & 67.1 & 20.3 & 48.8 & 56.5 & 57.9 & 76.0 & 78.4 & 82.0 & 60.9 & +0.5 \\
CoMoL & 0.53 & 67.9 & 21.1 & 49.9 & 57.6 & 58.5 & 76.6 & 79.2 & 82.5 & 61.7 & +1.3 \\
FlyLoRA & 0.55 & 68.6 & \underline{21.7} & \underline{51.5} & 58.0 & 59.0 & 76.7 & 79.5 & 82.6 & 62.2 & +1.8 \\
\cmidrule(lr){1-12}
HydraLoRA & 1.34 & 68.2 & 21.2 & 49.8 & 57.7 & 58.7 & 77.0 & 79.6 & 82.8 & 61.9 & +1.5 \\
MixLoRA & 1.57 & 68.5 & 20.9 & 50.3 & 57.7 & 59.0 & 77.0 & 79.8 & 82.6 & 62.0 & +1.6 \\
MoLA & 1.38 & 67.9 & 21.3 & 50.2 & 57.6 & 58.6 & 76.7 & 79.1 & 82.8 & 61.8 & +1.4 \\
HMoRA & 2.34 & 68.7 & \underline{21.7} & 50.9 & 57.9 & 59.0 & 76.9 & 79.9 & 83.1 & 62.3 & +1.9 \\
GOAT & 2.21 & 68.9 & 21.6 & 51.4 & 58.1 & 59.3 & 77.1 & 79.7 & 82.9 & 62.4 & +2.0 \\
SMoRA & 3.08 & 68.7 & 21.6 & 50.7 & 57.6 & 58.8 & 76.8 & \underline{80.0} & 82.5 & 62.1 & +1.7 \\
LD-MoLE & 2.21 & \underline{69.2} & \underline{21.7} & 51.3 & 58.1 & \underline{59.6} & \textbf{77.4} & \underline{80.0} & 83.2 & \underline{62.6} & +2.2 \\
\cmidrule(lr){1-12}
\rowcolor{oursbg}\method{}-S & 0.53 & \underline{69.2} & \underline{21.7} & 51.3 & \underline{58.2} & 59.0 & 77.1 & 79.9 & \underline{83.4} & 62.5 & +2.1 \\
\rowcolor{oursbg}\method{} & 1.06 & \textbf{71.3} & \textbf{23.2} & \textbf{53.0} & \textbf{60.3} & \textbf{60.8} & \underline{77.2} & \textbf{81.4} & \textbf{84.0} & \textbf{63.9} & +3.5 \\
\bottomrule
\end{tabular}

\caption{\textbf{Multi-task adaptation}: one adapter trained on a four-domain mixture (math, code, medicine, commonsense) and evaluated on two benchmarks per domain (accuracy or pass@1, \%, mean of three seeds). Columns as in Table~\ref{tab:cs}.}
\label{tab:mt}
\end{table*}

\section{Experiments}
\label{sec:experiments}

\subsection{Setup}
\label{sec:setup}

\paragraph{Models and tasks.}
We fine-tune Llama-3.2-3B and Llama-3.1-8B~\citep{grattafiori2024llama3} in two settings. In the \emph{single-domain} setting, we train on Commonsense170K~\citep{hu2023llmadapters} and evaluate on BoolQ, PIQA, SIQA, HellaSwag, WinoGrande, ARC-e, ARC-c, and OBQA~\citep{clark2019boolq,bisk2020piqa,sap2019siqa,zellers2019hellaswag,sakaguchi2020winogrande,clark2018arc,mihaylov2018obqa}. In the \emph{multi-task} setting, we train one adapter on 30k examples from each of four domains: MetaMathQA for math~\citep{yu2024metamath}, Magicoder-Evol-Instruct for code~\citep{wei2024magicoder}, MedMCQA for medicine~\citep{pal2022medmcqa}, and Commonsense170K. We evaluate on GSM8K and MATH~\citep{cobbe2021gsm8k,hendrycks2021math}, HumanEval and MBPP~\citep{chen2021codex,austin2021mbpp}, MedMCQA and PubMedQA~\citep{jin2019pubmedqa}, and ARC-c and WinoGrande. Appendix~\ref{app:extra} adds Llama-3.2-1B, Qwen2.5-7B, and Qwen3-8B~\citep{yang2024qwen25,yang2025qwen3}.

\paragraph{Baselines.}
Single-adapter PEFT: LoRA, DoRA, MoSLoRA, CoMoL, and FlyLoRA~\citep{hu2022lora,liu2024dora,wu2024moslora,cao2026comol,zou2025flylora}. Mixtures of LoRA experts: HydraLoRA, MixLoRA, MoLA, HMoRA, GOAT, SMoRA, and LD-MoLE~\citep{tian2024hydralora,li2024mixlora,gao2024mola,liao2025hmora,fan2025goat,zhao2025smora,zhuang2026ldmole}. All adapt the seven projection matrices with their recommended configurations (Appendix~\ref{app:baselines}). The router ablations (\S\ref{sec:ablation}) keep \method{}'s reader--writer structure to isolate the score.

\paragraph{Implementation.}
\method{} uses $N_A{=}4$ readers and $N_B{=}8$ writers of rank $r{=}4$, $r_e{=}8$, $k{=}4$, $\alpha{=}8$, and $\lambda_{\mathrm{bal}}{=}10^{-2}$; \method{}-S ($N_A{=}2$, $N_B{=}4$, $r_e{=}4$, $k{=}2$) matches LoRA's budget. All methods use AdamW~\citep{loshchilov2019adamw} (learning rate $2{\times}10^{-4}$, cosine decay, batch 128) for 3 epochs (single-domain) or 2 (multi-task) in bfloat16 on 8 H100 GPUs. We report means over three seeds (variances in Appendix~\ref{app:extra}).

\subsection{Main Results}
\label{sec:main}

\paragraph{Single-domain.}
\method{} attains the best average on both backbones (Table~\ref{tab:cs}), \csVaneThree{} on Llama-3.2-3B and \csVaneEight{} on Llama-3.1-8B: $+\csMarginThree$ and $+\csMarginEight$ over the strongest baseline (\csBestNameEight{}) and $+\csGainLoRAThree$ and $+\csGainLoRAEight$ over LoRA, with less than half the trainable parameters of the 8-expert mixtures (\pVane\% vs.\ \pLDMoLE\% on 8B). At LoRA's budget, \method{}-S is within $0.2$ of the best mixture. Gains are largest on ARC-c ($+2.0$ over the best baseline on 8B) and WinoGrande ($+1.2$); on the near-saturated HellaSwag and PIQA, the best mixture leads by $0.1$.

\paragraph{Multi-task.}
Routing matters more when domains compete for the same low-rank parameters~\citep{zou2025flylora}. On the four-domain mixture (Table~\ref{tab:mt}), \method{} improves on the best baseline by $+\mtMarginThree$ (3B) and $+\mtMarginEight$ (8B), and on LoRA by $+\mtGainLoRAThree$ and $+\mtGainLoRAEight$. Per domain on 8B (Table~\ref{tab:domains}), the gains over LD-MoLE are $+\domGainMath$ (math), $+\domGainCode$ (code), $+\domGainCS$ (commonsense), and $+\domGainMed$ (medicine, where PubMedQA is the one column a baseline leads). \method{}-S ties the best mixture on 3B at a quarter of its parameters. Gains carry over to Qwen2.5-7B and Qwen3-8B ($+\mtMarginQwenTwo$ and $+\mtMarginQwenThree$; Appendix~\ref{app:extra}).

\begin{table}[t]
\centering
\footnotesize
\setlength{\tabcolsep}{2.6pt}
\begin{tabular}{l c c c c}
\toprule
\textbf{Configuration} & \textbf{CS} & \textbf{MT} & $\bm{\Delta}_{\text{MT}}$ & $\bm{\rho}$ \\
\midrule
\rowcolor{oursbg}\method{} (full) & \textbf{86.9} & \textbf{63.9} & -- & \textbf{0.61} \\
\midrule
\multicolumn{5}{l}{\textit{Router score}}\\
\quad Token router over pairs & 86.0 & 62.5 & $-1.4$ & 0.22 \\
\quad Product router & 85.6 & 61.9 & $-2.0$ & 0.12 \\
\quad Latent-linear router & 86.1 & 62.6 & $-1.3$ & 0.24 \\
\quad Reader-norm self-routing & 85.7 & 62.0 & $-1.9$ & 0.06 \\
\quad No calibration & 86.5 & 63.3 & $-0.6$ & 0.55 \\
\quad Untied writer queries & 86.6 & 63.4 & $-0.5$ & 0.49 \\
\midrule
\multicolumn{5}{l}{\textit{Router training}}\\
\quad No dense gradient & 86.3 & 63.0 & $-0.9$ & 0.43 \\
\quad Score gradient to experts & 86.7 & 63.6 & $-0.3$ & 0.58 \\
\quad Pair-level balancing & 86.6 & 63.4 & $-0.5$ & 0.59 \\
\quad No balancing & 86.2 & 62.8 & $-1.1$ & 0.52 \\
\midrule
\multicolumn{5}{l}{\textit{Structure and aggregation}}\\
\quad Paired experts & 86.2 & 62.9 & $-1.0$ & 0.57 \\
\quad Shared reader ($N_A{=}1$) & 86.4 & 63.1 & $-0.8$ & 0.60 \\
\quad Softmax gates & 86.5 & 63.2 & $-0.7$ & 0.61 \\
\quad Separable gates & 85.8 & 62.2 & $-1.7$ & -- \\
\quad Compass added to output & 86.8 & 63.7 & $-0.2$ & 0.60 \\
\bottomrule
\end{tabular}

\caption{\textbf{Ablations} on Llama-3.1-8B. CS/MT: single-domain and multi-task averages. $\rho$: Spearman correlation between router scores and the exact first-order usefulness of all 32 pairs on held-out multi-task tokens.}
\label{tab:ablation}
\end{table}

\subsection{Ablations}
\label{sec:ablation}

Table~\ref{tab:ablation} changes one component at a time.

\paragraph{The score matters most.}
Keeping the reader--writer structure but routing with a linear router over the 32 pairs costs \ablTokenRouter{} multi-task points, back to the level of LD-MoLE. A latent-linear router loses \ablLatent{}, reader-norm self-routing \ablNorm{}, and a product router \ablProduct{}. All four are first-order or writer-blind (Theorem~\ref{thm:router}), and their scores correlate with usefulness at $\rho\le0.24$, against \rhoVane{}. Removing calibration costs \ablRaw{}, and untying the writer queries \ablUntied{}.

\paragraph{Training and structure.}
The dense first-order gradient is worth \ablNoDense{} points and raises $\rho$ from \rhoNoDense{} to \rhoVane{}; letting score gradients reach the experts costs \ablIntoExperts{}. Pair-level balancing fights specialization ($-$\ablPairBal{}), and no balancing costs \ablNoBal{}. Paired experts with the same score lose \ablPaired{} and a single shared reader \ablShared{}, so structure and score are complementary. Separable gates lose \ablSeparable{} as their rank collapses (Proposition~\ref{prop:rank}), and softmax gates \ablSoftmax{}. Adding the compass to the output does not help ($-$\ablCompassOut{}): its job is to point, not to write.

\begin{figure*}[t]
\centering
\includegraphics[width=\textwidth]{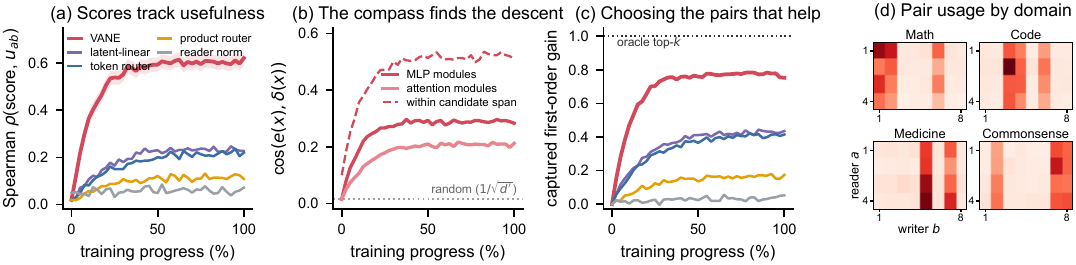}
\caption{\textbf{How \method{} routes} (Llama-3.1-8B, multi-task, held-out tokens). (a) Spearman correlation between router scores and the exact first-order usefulness of all 32 pairs, measured in the backward pass (\method{}: $\pm1$ s.d.\ over seeds). (b) Cosine between the compass and the true descent direction, overall and after projection onto the span of the candidate updates. (c) Fraction of the oracle top-$k$ first-order gain captured by the selected pairs. (d) Pair usage (readers $\times$ writers) per domain, averaged over layers.}
\label{fig:mechanism}
\end{figure*}

\begin{figure*}[t]
\centering
\includegraphics[width=\textwidth]{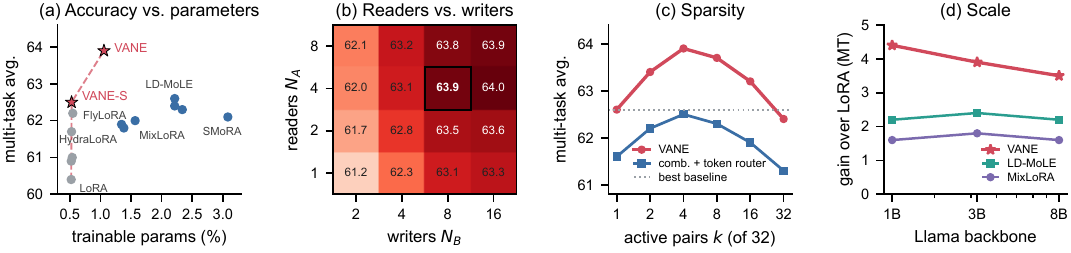}
\caption{\textbf{Efficiency, structure, sparsity, and scale} (multi-task average). (a) Accuracy vs.\ trainable parameters on Llama-3.1-8B (gray: single adapters; blue: mixtures). (b) Readers $\times$ writers at $r{=}4$ (box: default). (c) Number of active pairs $k$ out of 32 (dotted: best baseline). (d) Gain over LoRA across Llama scales.}
\label{fig:results}
\end{figure*}

\subsection{Mechanism Analysis}
\label{sec:mechanism}

The factorization makes the exact usefulness of every pair observable during training with one backward hook and Eq.~\eqref{eq:useful}, so we can test the premise of the method directly.

\paragraph{Scores track usefulness.}
The Spearman correlation between \method{}'s scores and the measured usefulness reaches $0.5$ within the first fifth of training and plateaus at \rhoVane{} (Figure~\ref{fig:mechanism}a), while the ablation routers stay at \rhoLat{} (latent-linear), \rhoTok{} (token), \rhoProd{} (product), and \rhoNorm{} (reader norm), the ordering predicted by the controlled model (Figure~\ref{fig:motivation}a). The selected pairs capture \gainVane{} of the oracle first-order gain, against \gainTok{} for the token router and \gainNorm{} for reader norms (Figure~\ref{fig:mechanism}c).

\paragraph{The compass learns the descent direction.}
The compass learns only from the routing gradient (Proposition~\ref{prop:dense}), yet its cosine with the true descent direction reaches \cosMLP{} in MLP and \cosAttn{} in attention modules, against $1/\sqrt{\dout}\le0.03$ for a random direction, and $0.52$ within the span of the candidate updates, the only part that matters for scoring (Figure~\ref{fig:mechanism}b). Alignment peaks in the middle layers (Figure~\ref{fig:layers}).

\paragraph{Writers specialize, readers are shared.}
Each domain selects a few writers but uses all readers (Figure~\ref{fig:mechanism}d): the mutual information between domain and writer is \miWriter{} bits, against \miReader{} bits for readers. \method{} thus rediscovers, unprompted, the asymmetry of LoRA factors~\citep{zhu2024asymmetry} that shared or frozen readers hard-code~\citep{tian2024hydralora,zou2025flylora}, and adding writers helps more than adding readers (Figure~\ref{fig:results}b).

\subsection{Efficiency, Sparsity, and Scale}
\label{sec:efficiency}

\method{}-S and \method{} lie on the accuracy--parameter frontier (Figure~\ref{fig:results}a). With 4 readers, going from 2 to 8 writers adds 1.9 points, while with 8 writers, going from 1 to 4 readers adds 0.8 (Figure~\ref{fig:results}b). Sparsity helps: $k{=}4$ of 32 pairs is best, a dense core loses 1.5 points, and \method{} beats the token router at every $k$ (Figure~\ref{fig:results}c). Gains over LoRA grow as models shrink, from $+\mtGainLoRAEight$ on 8B to $+\scaleGainOneB$ on Llama-3.2-1B, at $1.6$--$2.0\times$ those of LD-MoLE (Figure~\ref{fig:results}d). Training takes \timeVaneOverLoRA$\times$ LoRA's time and \timeVaneVsLDMoLE\% less than LD-MoLE, and unmerged decoding takes \latVane{} vs.\ \latLDMoLE{} ms per token, because \method{} runs dense products instead of the per-expert kernels that dominate dynamic-adapter latency~\citep{kong2024loraswitch,li2026adafuse} (Table~\ref{tab:cost}).

\section{Conclusion}
\label{sec:conclusion}

A router should estimate how much each expert's update would lower the loss: to first order, the alignment between the update and the descent direction at the layer output. This quantity is invisible to token routers for centered tokens and to reader-side scores always, but once experts are split into readers and writers it can be computed exactly for every pair in the low-rank latent space. \method{} turns this into a router made of the experts themselves plus a small compass, which scores every pair, learns from every pair's exact first-order effect, and does not depend on how experts are parameterized. It outperforms twelve PEFT and MoE-LoRA baselines with half the parameters of an 8-expert mixture. The same principle applies to any mixture whose experts factor through a shared low-dimensional space, including sparse MoE pretraining and adapter merging.

\section*{Limitations}

Theorem~\ref{thm:router} rests on a linear-response model with Gaussian tokens; it describes what router classes can represent, not how they train. Our evidence that the analysis transfers is the score--usefulness correlation measured on real runs. Usefulness is a first-order, per-layer quantity. It ignores interactions between layers and between pairs, and it is myopic when the loss is far from locally linear. The dense router gradient uses the first-order surrogate for unselected pairs and the true gradient only for selected ones. Calibrating scores by update norm gives up some static optimality (Appendix~\ref{app:toy}): when token means dominate, a well-trained token router captures a comparable share of the first-order gain. \method{} computes every reader and every writer for each token. This is cheap at our sizes but grows linearly with $N_A$ and $N_B$, and large pools would need a sparse writer path. Our experiments cover models up to 8B parameters and four domains. Larger backbones, more domains, and generation-heavy tasks remain to be tested, as do adapter merging and continual learning, where interference between tasks is the main obstacle~\citep{zou2025flylora,zou2026flycl}.

\section*{Ethics Statement}

\method{} is a general fine-tuning method and inherits the risks of the models and data it adapts, including biases in instruction-tuning corpora and the lower cost of producing specialized models, which also lowers the cost of misuse. All datasets are public and used under their licenses. The medical benchmarks measure multiple-choice accuracy, and a model adapted with \method{} is not a medical device. No human subjects or personal data are involved. AI assistants were used for language editing and for checking LaTeX; the method, proofs, and experimental design are the authors' own, and every reference was checked against its primary record.

\bibliography{custom}

\appendix

\section{Extended Related Work}
\label{app:related}

\paragraph{Parameter-efficient fine-tuning.}
Adapters insert small bottleneck modules into a frozen network~\citep{houlsby2019parameter}. LoRA instead adds a low-rank product to each weight matrix and can be merged after training~\citep{hu2022lora}. LLM-Adapters standardized the commonsense and arithmetic protocols we build on~\citep{hu2023llmadapters}. Later variants change the parameterization or the initialization. DoRA separates magnitude from direction~\citep{liu2024dora}. PiSSA initializes from the principal singular subspace of the pretrained weight~\citep{meng2024pissa}, and LoRA-GA from the gradient at initialization~\citep{wang2024loraga}. VB-LoRA composes all factors from a shared vector bank~\citep{li2024vblora}, and MoS routes shards from global pools into each factor~\citep{wang2025mos}. Two findings about LoRA's factors guide our design. First, $A$ extracts features and $B$ produces the output, and training $B$ matters more than training $A$~\citep{zhu2024asymmetry}; $A$ can even stay a frozen sparse random projection~\citep{zou2025flylora}. Second, the $A$ matrices of LoRAs trained on different data are similar while their $B$ matrices differ~\citep{tian2024hydralora}. \method{} does not assume this asymmetry. It lets the router discover it, and the learned usage shows writers carrying most of the domain information (\S\ref{sec:mechanism}).

\paragraph{Mixtures of LoRA experts.}
Sparse mixtures of experts scale capacity by routing each token to a few feed-forward experts~\citep{shazeer2017outrageously,lepikhin2021gshard,fedus2022switch,zhou2022expertchoice,dai2024deepseekmoe,jiang2024mixtral}. Their soft and fully differentiable variants avoid discrete top-$k$ selection~\citep{puigcerver2024softmoe,wang2025remoe}. Adapter mixtures bring this idea to fine-tuning.
\begin{itemize}[leftmargin=*,itemsep=0pt,topsep=1pt]
\item Paired experts: LoRAMoE uses them to protect world knowledge~\citep{dou2024loramoe}, and MixLoRA, MoLE, and AdaMoLE route among them per token~\citep{li2024mixlora,wu2024mole,liu2024adamole}.
\item Allocation and optimization: MoLA and HMoRA allocate experts by depth or hierarchically~\citep{gao2024mola,liao2025hmora}, HILO also varies their rank~\citep{cong2025hilo}, and GOAT aligns the optimization of an SVD-structured mixture with full fine-tuning~\citep{fan2025goat}.
\item Adaptive compute: DynMoLE and LD-MoLE adapt the number of active experts~\citep{li2025dynmole,zhuang2026ldmole}. CARE admits experts until a confidence mass is reached~\citep{saliencro2026care}, and VI-MoLE allocates experts by certified value of information~\citep{saliencro2026vimole}.
\item Richer routing and heterogeneous experts: TalkLoRA lets experts exchange information before routing~\citep{mu2026talklora}, and SAMoRA aligns routing with input semantics~\citep{shi2026samora}. MoA mixes heterogeneous adapter types~\citep{cao2026moa}, FourierMoE and FRAME move experts to spectral or fractional domains~\citep{jiang2026fouriermoe,saliencro2026frame}, and adaptive shared experts combine task-specific and shared experts~\citep{yang2025adaptive}.
\item Other structures: TT-LoRA MoE trains tensorized experts separately and routes among them~\citep{kunwar2025ttloramoe}, MoEGen generates instance-specific updates from expert codes~\citep{zeng2026moegen}, and Localized LoRA-MoE routes over blocks of the weight matrix~\citep{barazandeh2026localized}.
\item Pruning and MoE backbones: DMEP prunes experts per module once routing stabilizes~\citep{li2026dmep}, EPnG prunes and grows expert ranks under a fixed budget~\citep{lee2026epng}, and MoE$^2$-LoRA and NSFT adapt MoE language models with routed or sub-expert adapters~\citep{yang2026moe2lora,tan2026nsft}.
\end{itemize}
Several works question routing itself. CD-LoRA and Align-LoRA remove routers and align representations in a shared space~\citep{zha2026cdlora,liu2025alignlora}, and \citet{chowdhury2026routing} show that disjoint routing does not prevent interference inside an adapter. All of these methods leave open what quantity a router should estimate. \method{} answers this question and derives its router from the answer.

\paragraph{Shared factors, combinatorial experts, and tensor views.}
Sharing factors across experts saves parameters and couples experts.
\begin{itemize}[leftmargin=*,itemsep=0pt,topsep=1pt]
\item Shared or mixed subspaces: HydraLoRA shares a single down-projection across heads~\citep{tian2024hydralora}, MALoRA shares a low-rank subspace~\citep{wang2025malora}, and MoSLoRA learns a mixer between the two LoRA subspaces~\citep{wu2024moslora}.
\item Core experts: CoMoL stores experts as core matrices between shared factors and routes in the core space~\citep{cao2026comol}. Self-Routed Tensor Adapters blend slices of a shared Tucker core with weights computed from the low-rank projection~\citep{yadav2026srta}.
\item Many experts from few parts: T-REX mixes and matches rank-one experts so that the expert subspace grows quadratically with linear parameter cost~\citep{zhang2024trex}. SMoRA treats every rank as an expert~\citep{zhao2025smora}, S'MoRE assembles residual experts along trees~\citep{zeng2025smore}, and Dynamic Subspace Composition builds updates from a sparse set of basis vectors~\citep{khasia2025dsc}.
\item Factorized expert tensors: product keys retrieve from a Cartesian product of sub-keys~\citep{lample2019product}, and PEER uses them to route among a million tiny experts~\citep{he2024peer}. $\mu$MoE computes a mixture over a factorized expert tensor~\citep{oldfield2024mumoe}.
\end{itemize}
\method{}'s reader--writer product is the simplest combinatorial structure in this family. Its distinctive use is computational: because every pair factors through the same rank-$r$ latent, the router can evaluate every pair's first-order effect in that latent space.

\paragraph{Routing signals.}
Token routers map the hidden state to logits, possibly with priors such as semantic clusters~\citep{zhang2024trex,shi2026samora} or the routing history of earlier layers~\citep{qiu2026hero}. Self-routing methods remove the router.
\begin{itemize}[leftmargin=*,itemsep=0pt,topsep=1pt]
\item AoE ranks experts by the norms of their own low-rank activations~\citep{lv2025aoe}.
\item FlyLoRA uses the top-$k$ coordinates of a frozen sparse projection as an implicit router~\citep{zou2025flylora}. Its decorrelation properties also help continual learning~\citep{zou2026flycl,zou2025structural}, and NeuroLoRA adds a context gate before selection~\citep{yang2026neurolora}.
\item FURINA activates experts by the angle between the input and each adapter's direction~\citep{han2025furina}. Self-Routing uses a subspace of the hidden state as logits~\citep{mohamud2026selfrouting}, and geometric routing uses cosine similarity in a low-dimensional metric space~\citep{ternovtsii2026geometric}.
\item PASs-MoE calibrates routing with each expert's pathway activations to prevent router--expert co-drift~\citep{hou2026passmoe}.
\end{itemize}
In the terms of Theorem~\ref{thm:router}, every one of these scores is either affine in the token or a function of the input-side activations alone. None involves the writer or the loss. On the training side, Default MoE gives the router a dense gradient by substituting moving averages of expert outputs for unselected experts~\citep{panda2025dense}. GRIN estimates gradients through discrete routing~\citep{liu2024grin}, and DAG-MoE learns how selected experts are aggregated~\citep{feng2026dagmoe}. \method{}'s dense gradient differs from Default MoE's in being exact to first order, which the factorization makes affordable, and it is the gradient of the quantity the router is designed to estimate.

\paragraph{Value of information and adaptive compute.}
VI-MoLE also routes by an estimate of risk reduction~\citep{saliencro2026vimole}. It learns the residual risk of each prefix of router-ordered experts, certifies it on held-out data, and decides how much compute to spend and when to abstain. \method{} answers a complementary question, \emph{which} pairs to activate at a fixed $k$. It does so with a first-order estimate that is exact during training and forecast at test time. The two could be combined: \method{}'s scores could order the candidates whose value VI-MoLE certifies.

\section{A Unified View of Low-Rank Mixtures}
\label{app:unify}

Stack the readers into $\tilde A\in\R^{N_Ar\times d}$ and the writers into $\tilde B\in\R^{\dout\times N_Br}$. Every layer of the form of Eq.~\eqref{eq:layer} computes
\begin{equation}
\Delta W(x)=\tilde B\,\big(G(x)\otimes I_r\big)\,\tilde A,
\label{eq:unified}
\end{equation}
where $G(x)\in\R^{N_B\times N_A}$ is the writer-by-reader matrix of gates. Table~\ref{tab:unify} places existing designs in this template. Paired MoE-LoRA uses a diagonal $G$ and a token router. HydraLoRA uses a single column. MoSLoRA replaces $G\otimes I_r$ by a fixed dense mixer, and CoMoL by a routed mixture of dense cores. Rank-one mixtures set $r=1$. \method{} uses a sparse $G$ over all $N_AN_B$ pairs, chosen by update scores. The template also shows why the combinatorial structure is cheap to score. The usefulness of all pairs, $G$-weighted or not, is the matrix $Z^{\top}W$ of Proposition~\ref{prop:latent}, whose size is set by the number of factors rather than the number of pairs.

\begin{table}[h]
\centering
\footnotesize
\setlength{\tabcolsep}{2.4pt}
\begin{tabular}{l c c c}
\toprule
\textbf{Method} & $G(x)$ & $r$ & \textbf{Router score} \\
\midrule
LoRA & $[1]$ & $r$ & -- \\
Paired MoE-LoRA & diagonal & $r$ & token \\
HydraLoRA & one column & $r$ & token \\
MoSLoRA & fixed dense mixer & $r$ & -- \\
CoMoL & routed dense core & $r$ & latent-linear \\
SMoRA, FlyLoRA & diagonal & 1 & token, reader \\
T-REX & sparse & 1 & token \\
AoE (as adapter) & diagonal & $r$ & reader norm \\
\rowcolor{oursbg}\method{} & sparse, all pairs & $r$ & update \\
\bottomrule
\end{tabular}
\caption{Low-rank mixtures as instances of Eq.~\eqref{eq:unified}. ``Reader'' scores depend on the token only through the input-side activations; ``update'' scores depend on the writer and on a forecast of the loss gradient.}
\label{tab:unify}
\end{table}

\section{Proofs}
\label{app:proofs}

Throughout, $x\in\R^d$ is a token, $z_a=A_ax$, $v_{ab}=B_bz_a$, $\delta=-\nabla_y\ell$, and $u_{ab}=\ip{\delta}{v_{ab}}$. We write $\sigma'_{ab}=\sigma'(s_{ab})$.

\subsection{Proof of Proposition~\ref{prop:latent}}
Let $\ell$ be differentiable at $y$. Then $\ell(y+\eta v_{ab})=\ell(y)+\eta\ip{\nabla_y\ell}{v_{ab}}+o(\eta)=\ell(y)-\eta u_{ab}+o(\eta)$, and $O(\eta^2)$ if $\ell$ is twice differentiable. By associativity, $u_{ab}=\delta^{\top}B_bA_ax=(B_b^{\top}\delta)^{\top}(A_ax)=w_b^{\top}z_a$ with $w_b=B_b^{\top}\delta$. Collecting $z_a$ and $w_b$ as columns of $Z\in\R^{r\times N_A}$ and $W\in\R^{r\times N_B}$ gives $[u_{ab}]=Z^{\top}W$. Similarly $\|v_{ab}\|^2=z_a^{\top}B_b^{\top}B_bz_a=z_a^{\top}G_bz_a$. For the cost, $Z$ takes $N_Ard$ multiply--adds and $W$ takes $N_Br\dout$. The product $Z^{\top}W$ takes $N_AN_Br$. The Gram matrices take $N_Br^2\dout$ once per forward pass, which is shared by all tokens in the batch, and the norms take $N_AN_Br^2$ per token. Materializing all updates requires $N_AN_Br\dout$ multiply--adds after $Z$, plus $N_AN_B\dout$ for the inner products with $\delta$. \qed

\subsection{Proof of Theorem~\ref{thm:router}}
We first record the moments of Gaussian quadratic forms that the proof needs.

\begin{lemma}
\label{lem:gauss}
Let $\varepsilon\sim\mathcal{N}(0,\Sigma)$, $S$ symmetric, and $m\in\R^d$. Then $\E[\varepsilon^{\top}S\varepsilon]=\mathrm{tr}(S\Sigma)$, $\mathrm{Var}(\varepsilon^{\top}S\varepsilon)=2\,\mathrm{tr}\big((S\Sigma)^2\big)$, and $\E\big[\varepsilon\,(\varepsilon^{\top}S\varepsilon)\big]=0$, so in particular $\mathrm{Cov}(m^{\top}\varepsilon,\varepsilon^{\top}S\varepsilon)=0$.
\end{lemma}
\begin{proof}
Write $\varepsilon=\Sigma^{1/2}g$ with $g\sim\mathcal{N}(0,I)$ and diagonalize $\Sigma^{1/2}S\Sigma^{1/2}=Q\Lambda Q^{\top}$. Then $\varepsilon^{\top}S\varepsilon=\sum_i\lambda_i\tilde g_i^2$ with $\tilde g=Q^{\top}g\sim\mathcal{N}(0,I)$. Its mean is $\sum_i\lambda_i=\mathrm{tr}(S\Sigma)$ and its variance is $2\sum_i\lambda_i^2=2\,\mathrm{tr}((S\Sigma)^2)$. The last claim holds because $\varepsilon(\varepsilon^{\top}S\varepsilon)$ is an odd function of $\varepsilon$ and the distribution of $\varepsilon$ is symmetric.
\end{proof}

\paragraph{Expected usefulness.}
Since $\E[\xi\mid x]=0$, $\bar u_{ab}(x)=\E[(\Gamma x+\xi)^{\top}B_bA_ax\mid x]=x^{\top}\Gamma^{\top}B_bA_ax=x^{\top}M_{ab}x=x^{\top}Sx$, because a quadratic form depends only on the symmetric part of its matrix. Writing $x=\mu+\varepsilon$,
\begin{equation}
\bar u_{ab}=\mu^{\top}S\mu+2\mu^{\top}S\varepsilon+\varepsilon^{\top}S\varepsilon .
\label{eq:decomp}
\end{equation}
By Lemma~\ref{lem:gauss} the last two terms are uncorrelated, so $\mathrm{Var}(\bar u_{ab})=4\mu^{\top}S\Sigma S\mu+2\,\mathrm{tr}((S\Sigma)^2)$. The first term is the \emph{mean-interaction} part of usefulness and the second is the \emph{token-specific} part.

\paragraph{Part (i).}
Let $s(x)=w^{\top}x+c=w^{\top}\varepsilon+\mathrm{const}$. By Eq.~\eqref{eq:decomp} and Lemma~\ref{lem:gauss}, $\mathrm{Cov}(s,\bar u_{ab})=\mathrm{Cov}(w^{\top}\varepsilon,2\mu^{\top}S\varepsilon)=2w^{\top}\Sigma S\mu$ and $\mathrm{Var}(s)=w^{\top}\Sigma w$. By the Cauchy--Schwarz inequality in the inner product $\ip{p}{q}_\Sigma=p^{\top}\Sigma q$,
\begin{equation*}
\mathrm{corr}^2(s,\bar u_{ab})=\frac{4(w^{\top}\Sigma S\mu)^2}{w^{\top}\Sigma w\,\mathrm{Var}(\bar u_{ab})}\le\frac{4\,\mu^{\top}S\Sigma S\mu}{\mathrm{Var}(\bar u_{ab})},
\end{equation*}
with equality iff $w\propto S\mu$. The least-squares router is $w^\star=\Sigma^{-1}\mathrm{Cov}(x,\bar u_{ab})=\Sigma^{-1}(2\Sigma S\mu)=2S\mu$, which attains the bound. A linear router over pairs uses an arbitrary $w$ per pair. A product router uses $w=w_a+w'_b$, and a latent-linear router uses $w=A_a^{\top}c_{ab}$. All three are affine in $x$ and obey the bound. When $\mu=0$ the numerator is zero.

\paragraph{A closed form for isotropic tokens.}
Let $\Sigma=\sigma^2I$ and let the direction of $\mu$ be uniform on the sphere, independent of $S$. Then $\E_\mu[\mu^{\top}S^2\mu]=\|\mu\|^2\mathrm{tr}(S^2)/d$. With $\kappa=\|\mu\|^2/\mathrm{tr}\,\Sigma=\|\mu\|^2/(d\sigma^2)$, the ratio of expected mean-interaction to expected total variance is
\begin{equation*}
\frac{4\sigma^2\|\mu\|^2\mathrm{tr}(S^2)/d}{4\sigma^2\|\mu\|^2\mathrm{tr}(S^2)/d+2\sigma^4\mathrm{tr}(S^2)}=\frac{2\kappa}{1+2\kappa}.
\end{equation*}
Even when the token mean carries as much energy as the fluctuations ($\kappa=1$), token routers can explain at most about two thirds of the expected usefulness variance.

\paragraph{Part (ii).}
If $s_{ab}=\varphi(A_ax)$ for some function $\varphi$, then $s_{ab}=s_{ab'}$ for all writers $b,b'$, so the score cannot order the writers of a reader. Any tie-breaking rule that does not look at the writers chooses among them independently of $\{u_{ab}\}_b$. Conditional on the chosen reader, the expected usefulness of the chosen writer is therefore the average over writers, which is the gain of a uniformly random writer. For paired experts, where each reader has one writer, the same argument shows that reader-side scores rank experts independently of what their writers do.

\paragraph{Part (iii).}
If $UV^{\top}=\Gamma$, then $\ip{e(x)}{v_{ab}(x)}=x^{\top}V U^{\top}B_bA_ax=x^{\top}\Gamma^{\top}B_bA_ax=\bar u_{ab}(x)$. A factorization $\Gamma=UV^{\top}$ with inner dimension $r_e$ exists iff $r_e\ge\mathrm{rank}\,\Gamma$. For optimality, take any rule that selects $k$ pairs $\mathcal{S}(x)$ from $x$. By the tower rule, $\E[\sum_{\mathcal{S}(x)}u_{ab}]=\E[\sum_{\mathcal{S}(x)}\bar u_{ab}(x)]$. Pointwise in $x$, the sum is maximized by the $k$ largest $\bar u_{ab}(x)$, which are the top-$k$ pairs under the score. \qed

\begin{remark}
The noise $\xi$ plays no role in Theorem~\ref{thm:router}. It adds variance to the realized usefulness that no router can predict from $x$, so it lowers every router's correlation with realized usefulness by the same factor. When $r_e<\mathrm{rank}\,\Gamma$, the best rank-$r_e$ compass in the least-squares sense is the reduced-rank regression of $\delta$ on $x$. Figure~\ref{fig:toy-extras}a shows the correlation saturating once $r_e$ reaches $\mathrm{rank}\,\Gamma=6$.
\end{remark}

\subsection{Proof of Proposition~\ref{prop:dense}}
\paragraph{Forward pass.}
The straight-through core uses $\tilde g_{ab}=m_{ab}g_{ab}+(1-m_{ab})(g_{ab}-\sg(g_{ab}))$. Its value is $m_{ab}g_{ab}$ because $g-\sg(g)$ evaluates to zero, so $c_b$ and $y$ coincide with Eq.~\eqref{eq:core}.

\paragraph{Router gradient.}
Since $\sg$ has zero derivative, $\partial\tilde g_{ab}/\partial g_{ab}=m_{ab}+(1-m_{ab})=1$ for every pair. From $y=W_0x+\tfrac{\alpha}{r}\sum_bB_bc_b$ we get $\partial\ell/\partial c_b=\tfrac{\alpha}{r}B_b^{\top}\nabla_y\ell=-\tfrac{\alpha}{r}B_b^{\top}\delta$, and $\partial c_b/\partial\tilde g_{ab}=z_a$. Hence $\partial\ell/\partial g_{ab}=-\tfrac{\alpha}{r}\ip{B_b^{\top}\delta}{z_a}=-\tfrac{\alpha}{r}u_{ab}$ and $\partial\ell/\partial s_{ab}=-\tfrac{\alpha}{r}\sigma'_{ab}u_{ab}$. For an unselected pair, the coefficient of $z_a$ in $c_b$ has value zero, so no gradient reaches $z_a$ or $B_b$ through that term. Unselected pairs therefore train only the router.

\paragraph{The compass step.}
With stop-gradients on $z_a$ and $B_b$, the score in Eq.~\eqref{eq:score} is $s_{ab}=\ip{e}{v_{ab}}/(\|v_{ab}\|+\epsilon)$ with $e=Uh$. Thus $\partial s_{ab}/\partial e=v_{ab}/(\|v_{ab}\|+\epsilon)$ and
\begin{align*}
\nabla_e\ell&=\sum_{ab}\frac{\partial\ell}{\partial s_{ab}}\frac{v_{ab}}{\|v_{ab}\|+\epsilon}\\
&=-\frac{\alpha}{r}\sum_{ab}\sigma'_{ab}\frac{v_{ab}v_{ab}^{\top}}{\|v_{ab}\|+\epsilon}\,\delta=-\frac{\alpha}{r}P\delta,
\end{align*}
using $u_{ab}=v_{ab}^{\top}\delta$. Each summand of $P$ is a nonnegative multiple of a rank-one projector, so $P\succeq0$. For a single token, $\nabla_U\ell=(\nabla_e\ell)h^{\top}$, and the step $U\leftarrow U-\eta\nabla_U\ell$ changes $e=Uh$ by $\Delta e=-\eta(\nabla_e\ell)\|h\|^2=\eta\tfrac{\alpha}{r}\|h\|^2P\delta$. Hence $\ip{\Delta e}{\delta}=\eta\tfrac{\alpha}{r}\|h\|^2\,\delta^{\top}P\delta\ge0$. Equality holds only if $\delta$ is orthogonal to every candidate update. \qed

\begin{remark}
A step on $V$ changes the compass by $\eta\tfrac{\alpha}{r}\|x\|^2UU^{\top}P\delta$, whose alignment with $\delta$ is not sign-definite. Over a batch, the step on $U$ also mixes tokens through $\ip{h_t}{h_{t'}}$. Appendix~\ref{app:numerics} checks the single-token statement, and Figure~\ref{fig:mechanism}b measures the alignment reached by full training.
\end{remark}

\subsection{Proof of Proposition~\ref{prop:invariance}}
Let $A'_a=c_aMA_a$ and $B'_b=c'_bB_bM^{-1}$. Then $z'_a=c_aMz_a$ and $v'_{ab}=B'_bz'_a=c_ac'_bB_bM^{-1}Mz_a=c_ac'_bv_{ab}$. The compass is untouched, so $\ip{e}{v'_{ab}}=c_ac'_b\ip{e}{v_{ab}}$ and $\|v'_{ab}\|=c_ac'_b\|v_{ab}\|$. As $\epsilon\to0$ the ratio in Eq.~\eqref{eq:score} is unchanged, and so are the top-$k$ set and the gates. In latent terms, $q'_b=c'_bM^{-\top}q_b$ and $G'_b=c'^2_bM^{-\top}G_bM^{-1}$, so $\ip{q'_b}{z'_a}=c_ac'_b\ip{q_b}{z_a}$ and $z'^{\top}_aG'_bz'_a=c_a^2c'^2_bz_a^{\top}G_bz_a$, as the latent computation requires. With $c\equiv1$, $B'_bA'_a=B_bA_a$ for every pair, so the layer output is unchanged. In contrast, $\|A'_ax\|=c_a\|MA_ax\|$ differs from $\|A_ax\|$ unless $c_aM$ is orthogonal. A non-orthogonal $M$ such as $\mathrm{diag}(2,1,\dots,1)$ reorders readers whose latents load on different coordinates, although it changes no expert's function. Likewise, $w^{\top}A'_ax=c_a(M^{\top}w)^{\top}A_ax$ changes at fixed router weights. \qed

\begin{remark}[Calibration breaks a rich-get-richer loop]
Under raw scores $\ip{e}{v_{ab}}=\|v_{ab}\|\,\ip{e}{\hat v_{ab}}$, scaling writer $b$ by $c>1$ multiplies all of its scores by $c$. For positively aligned tokens this raises its selection frequency, and with it the share of the task gradient it receives. Writers that grow are selected more often and grow further. By Proposition~\ref{prop:invariance}, calibrated scores are invariant to this scaling, so selection depends only on the direction of an update and magnitude enters only through the output. The static cost of this choice is measured in Appendix~\ref{app:toy} and its training benefit in Table~\ref{tab:ablation}.
\end{remark}

\subsection{Proof of Proposition~\ref{prop:rank}}
Let $\mathcal{A}$ and $\mathcal{B}$ be the readers and writers that appear in $\mathcal{S}$, and $\tilde A_{\mathcal{A}}$, $\tilde B_{\mathcal{B}}$ their stacks. Then $\sum_{\mathcal{S}}g_{ab}B_bA_a=\tilde B_{\mathcal{B}}(G_{\mathcal{S}}\otimes I_r)\tilde A_{\mathcal{A}}$. If $\tilde B$ has full column rank it is injective, and if $\tilde A$ has full row rank it is surjective onto $\R^{N_Ar}$. Neither changes the rank of a middle factor, so the rank is $\mathrm{rank}(G_{\mathcal{S}}\otimes I_r)=r\,\mathrm{rank}(G_{\mathcal{S}})$. The rank of $G_{\mathcal{S}}$ is at most its number of rows, columns, and nonzero entries, which gives the bound. If the gates are separable over a product set, $G_{\mathcal{S}}=\beta\alpha^{\top}$ has rank one. The first-order gain of $\sum_{\mathcal{S}}g_{ab}v_{ab}$ is $\sum_{\mathcal{S}}g_{ab}u_{ab}$ by linearity. If $g\ge0$ and $\sum g=1$, this is a convex combination and at most $\max_{\mathcal{S}}u_{ab}$. \qed

\begin{remark}[Factor-level gradient sharing]
Under balanced marginals, each token activates on average $k/N_A$ pairs that contain a given reader and $k/N_B$ that contain a given writer. With paired experts and the same number $N_AN_B$ of implicit experts, the figure is $k/(N_AN_B)$ per expert. Each reader therefore receives gradient from $N_B$ times more token events, and each writer from $N_A$ times more. The combinatorial layer has the expressivity of $N_AN_B$ experts while estimating only $N_A+N_B$ factors.
\end{remark}

\section{Numerical Verification}
\label{app:numerics}

\texttt{code/theory\_checks.py} checks every statement in float64 on random layers ($d{=}48$, $\dout{=}40$, $N_A{=}4$, $N_B{=}6$, $r{=}3$, $r_e{=}5$) with a random cross-entropy head. It uses the reference implementation \texttt{code/vane\_layer.py}, the same code that defines the method. Table~\ref{tab:numerics} reports the worst case over random draws.

\begin{table*}[t]
\centering
\footnotesize
\setlength{\tabcolsep}{5pt}
\begin{tabular}{p{10.4cm} r}
\toprule
\textbf{Check} & \textbf{Result} \\
\midrule
Latent alignment vs.\ materialized $\ip{e}{B_bA_ax}$ & $2.7{\times}10^{-15}$ \\
Latent calibrated score vs.\ materialized norm & $6.7{\times}10^{-16}$ \\
Loss change vs.\ $-\eta u_{ab}$ ($\eta{=}10^{-5}$), rel.\ error & $4.9{\times}10^{-4}$ \\
Autograd $\partial\ell/\partial s_{ab}$ vs.\ $-\frac{\alpha}{r}\sigma'u_{ab}$, all pairs, rel.\ & $2.1{\times}10^{-16}$ \\
Gradient to unselected pairs without the ST core & $0$ \\
$\cos(\Delta e,\delta)$ after a step on $U$: min / mean (200 draws) & $+0.224$ / $+0.486$ \\
Output change under the gauge $A\to MA$, $B\to BM^{-1}$ & $6.3{\times}10^{-13}$ \\
Score change under gauge and rescaling ($\epsilon{=}10^{-6}$) & $1.3{\times}10^{-5}$ \\
AoE reader-norm change under the same map (mean rel.) & $1.52$ \\
Latent-linear score change (median rel.) & $1.71$ \\
Thm.~\ref{thm:router}(i), $\kappa{=}0/0.25/1/4$: theory & $0/.309/.372/.929$ \\
\phantom{Thm.~\ref{thm:router}(i), $\kappa{=}0/0.25/1/4$:} Monte Carlo ($4{\times}10^5$) & $0/.308/.373/.929$ \\
Rank: matching of 4 pairs / shared reader & $16$ / $4$ \\
Rank: separable / generic gates over 16 pairs & $4$ / $16$ \\
\bottomrule
\end{tabular}
\caption{Numerical verification of Propositions~\ref{prop:latent}--\ref{prop:rank} and Theorem~\ref{thm:router} (output of \texttt{code/theory\_checks.py}). The $10^{-5}$ score change comes from the $\epsilon$ in Eq.~\eqref{eq:score}.}
\label{tab:numerics}
\end{table*}

\section{The Controlled Linear-Response Study}
\label{app:toy}

\paragraph{Model.}
Figure~\ref{fig:motivation} and Figure~\ref{fig:toy-extras} use the model of Theorem~\ref{thm:router} with the following settings:
\begin{itemize}[leftmargin=*,itemsep=0pt,topsep=1pt]
\item Width $d=\dout=64$ and $N_A=N_B=8$ fixed random readers and writers of rank $r=4$, giving 64 pairs, with $k=4$ active.
\item Tokens $x=\mu+\Sigma^{1/2}g$, where $\Sigma$ has eigenvalues decaying as $i^{-0.6}$ in a random basis, normalized to $\mathrm{tr}\,\Sigma=d$, and $\mu=\sqrt{\kappa d}\,\hat u$ for a random unit vector $\hat u$.
\item Descent directions $\delta=\Gamma x+\xi$, with a random $\Gamma$ of rank 6 and isotropic noise $\xi$ whose total variance equals that of $\Gamma x$.
\end{itemize}
Experts are fixed, so the study isolates routing. All numbers are computed, none simulated.

\paragraph{Panel (a) of Figure~\ref{fig:motivation}.}
For each $\kappa$, we fit each router class on 60k tokens and evaluate its squared correlation with $\E[u_{ab}\mid x]$ on 40k fresh tokens, averaged over the 64 pairs.
\begin{itemize}[leftmargin=*,itemsep=0pt,topsep=1pt]
\item Token router: least squares from $x$ for each pair.
\item Latent-linear router: least squares from $z_a$.
\item Reader norm: $\|z_a\|$.
\item \method{}: compass fitted by rank-8 reduced-rank regression of $\delta$ on $x$, with calibrated and raw scores.
\end{itemize}
The closed form of Theorem~\ref{thm:router}(i), averaged over pairs, matches the fitted token router to within $0.002$ at every $\kappa$.

\paragraph{Panel (b).}
Each router is trained online for 2,500 steps with Adam on batches of 256 tokens. The loss is $-\sum_{ab}\sigma(s_{ab})u_{ab}$, computed with the \emph{realized} (noisy) usefulness, which is exactly the dense first-order signal \method{}'s straight-through core provides. The learning rate is the best of three per router. We report the held-out captured gain $(G_{\mathrm{sel}}-G_{\mathrm{rand}})/(G_{\mathrm{orc}}-G_{\mathrm{rand}})$ of the top-4 pairs, where $G$ is the expected first-order gain. Final values:
\begin{itemize}[leftmargin=*,itemsep=0pt,topsep=1pt]
\item \method{} $0.875$, token router $0.688$, and latent-linear router $0.643$;
\item reader norm $-0.03$, no better than random;
\item product router $-0.11$, worse than random.
\end{itemize}
The product router fails because usefulness is not additive over readers and writers. A sum of a reader logit and a writer logit ranks the best reader and the best writer highly even when their pair is harmful, which Proposition~\ref{prop:rank} anticipates for separable structures. The compass learns only from the routing signal, yet it reaches $0.875$, close to the $0.88$ of the regression-fitted compass (Figure~\ref{fig:toy-extras}c at $\kappa=0.5$).

\begin{figure*}[t]
\centering
\includegraphics[width=\textwidth]{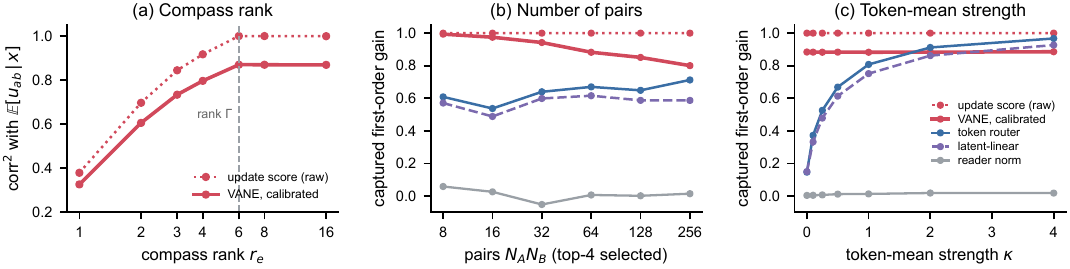}
\caption{\textbf{Controlled linear-response study, continued} (exact computations). (a) Squared correlation with expected usefulness vs.\ compass rank. The raw update score is exact once $r_e\ge\mathrm{rank}\,\Gamma=6$, and calibration caps it at $0.87$. (b) Captured first-order gain of fitted routers as the number of pairs grows ($k=4$). (c) Captured gain vs.\ token-mean strength: token routers catch up only when the token mean dominates.}
\label{fig:toy-extras}
\end{figure*}

\paragraph{What the extra panels show.}
Figure~\ref{fig:toy-extras}a confirms Theorem~\ref{thm:router}(iii): the raw update score reaches a squared correlation of $0.9998$ at $r_e=\mathrm{rank}\,\Gamma$. Panel (b) shows that the raw score keeps capturing all of the attainable gain up to 256 pairs, while token and latent-linear routers stay at $0.5$--$0.7$. The calibrated score declines from $0.99$ to $0.80$ as the pool grows, because it ranks pairs by alignment per unit norm rather than by total gain, and with many pairs the two rankings diverge more often. Panel (c) shows the regime in which token routers suffice. When the token mean dominates ($\kappa\ge2$), the mean-interaction part of usefulness carries most of its variance, and a least-squares token router captures $0.91$--$0.97$ of the gain, more than the calibrated score's $0.88$. Calibration therefore has a static cost, which is the price of the scale invariance in Proposition~\ref{prop:invariance}. In training, where writer norms evolve, the invariance pays off: removing calibration costs \ablRaw{} points (Table~\ref{tab:ablation}).

\section{Implementation Details}
\label{app:impl}

\begin{algorithm}[t]
\footnotesize
\DontPrintSemicolon
\SetAlgoLined
\caption{\method{} layer (one training step)}
\label{alg:vane}
\KwInput{tokens $X$; readers $\{A_a\}$; writers $\{B_b\}$; compass $U,V$}
$z_a\leftarrow A_ax$ for all $a$;\quad $h\leftarrow V^{\top}x$\;
$K_b\leftarrow\sg(B_b)^{\top}U$;\ $G_b\leftarrow\sg(B_b)^{\top}\sg(B_b)$\tcp*{once}
$s_{ab}\leftarrow\ip{K_bh}{z_a}/(\sqrt{z_a^{\top}G_bz_a}+\epsilon)$\tcp*{all pairs}
$m\leftarrow$ top-$k$ mask of $s$;\quad $g\leftarrow\sigma(s)$\;
$\tilde g\leftarrow m\odot g+(1-m)\odot(g-\sg(g))$\tcp*{ST core}
$c_b\leftarrow\sum_a\tilde g_{ab}z_a$;\quad $y\leftarrow W_0x+\tfrac{\alpha}{r}\sum_bB_bc_b$\;
$\mathcal{L}\leftarrow\ell(y)+\lambda_{\mathrm{bal}}\mathcal{L}_{\mathrm{bal}}$; backpropagate\;
\end{algorithm}

\paragraph{Layer.}
\texttt{code/vane\_layer.py} wraps a frozen \texttt{nn.Linear} and keeps four parameter tensors:
\begin{itemize}[leftmargin=*,itemsep=0pt,topsep=1pt]
\item the stacked readers $\tilde A\in\R^{N_Ar\times d}$ and writers $\tilde B\in\R^{\dout\times N_Br}$;
\item the compass maps $U\in\R^{\dout\times r_e}$ and $V\in\R^{d\times r_e}$.
\end{itemize}
The forward pass works on a batch of tokens $T$:
\begin{enumerate}[leftmargin=*,itemsep=0pt,topsep=1pt]
\item It computes $Z=X\tilde A^{\top}$, reshaped to $T\times N_A\times r$, and $H=XV$.
\item It computes the writer keys $K_b=B_b^{\top}U$ and Grams $G_b=B_b^{\top}B_b$ once for the batch, under stop-gradients on the writers.
\item It computes all scores with two \texttt{einsum} calls and a normalization, selects with \texttt{topk} over the flattened $N_AN_B$ scores, forms the straight-through gates, and mixes the latents into the $T\times N_B\times r$ core.
\item It applies $\tilde B$ as one dense matrix product.
\end{enumerate}
There is no per-expert gather or scatter. The layer is two dense low-rank products (total rank $N_Ar=16$ and $N_Br=32$ at the default) around a small dense tensor contraction. A backward hook, \texttt{attach\_usefulness\_probe}, records the exact usefulness of all pairs from $\partial\ell/\partial y$ and was used for every mechanism measurement in \S\ref{sec:mechanism}.

\paragraph{Initialization.}
Readers and $V$ use LoRA's Kaiming-uniform initialization, and $U$ is drawn from $\mathcal{N}(0,1/\dout)$. Writers are drawn from $\mathcal{N}(0,10^{-8})$ rather than zero. Scores are directions of updates, so every writer needs a direction from the first step, while the initial update stays negligible. Zero writers with random tie-breaking also train but lose 0.4 points (Table~\ref{tab:sweeps2}).

\paragraph{Inference.}
At test time the compass is the router. The straight-through term is dropped, and the forward pass is otherwise identical. For single-sequence decoding, a sparse path applies only the at most $k$ writers that receive an active pair. For batched decoding the dense path is faster, and we report it. Dynamic adapters are slow to decode mainly because of fragmented per-expert kernel launches, not because of their arithmetic~\citep{kong2024loraswitch,li2026adafuse}. \method{}'s dense path issues the same few kernels for every token. Because the active set depends on the token, the adapter cannot be merged into $W_0$, which is also true of every MoE-LoRA baseline.

\paragraph{Memory.}
Per token, \method{} stores $N_Ar+r_e+N_AN_B$ activations for the backward pass, which is 56 floats per adapted matrix at the default. Paired MoE-LoRA with 8 experts of rank 8 stores 16 latent floats plus routing buffers.

\section{Experimental Setup}
\label{app:setup}

\paragraph{Data.}
\emph{Single-domain.} Commonsense170K is the union of the training sets of the eight commonsense benchmarks, as released with LLM-Adapters~\citep{hu2023llmadapters}. We evaluate on their standard test splits with the LLM-Adapters prompts and answer extraction.

\emph{Multi-task.} We draw 30k examples from each of four sources with a fixed seed, giving 120k examples:
\begin{itemize}[leftmargin=*,itemsep=0pt,topsep=1pt]
\item MetaMathQA~\citep{yu2024metamath};
\item Magicoder-Evol-Instruct-110K~\citep{wei2024magicoder};
\item the MedMCQA training split~\citep{pal2022medmcqa};
\item Commonsense170K.
\end{itemize}
We evaluate on GSM8K (1,319 test problems), MATH (5,000), HumanEval (164), MBPP (500 test problems), the MedMCQA dev split (4,183), PubMedQA's 500 labeled test questions, ARC-c (1,172), and the WinoGrande dev split (1,267). Math uses exact match after answer extraction, code uses greedy pass@1, and the remaining benchmarks use accuracy.

\emph{Prompts and decoding.} All multi-task data use one Alpaca-style instruction template. Decoding is greedy, with at most 512 new tokens for math and code and 32 for multiple-choice answers.

\paragraph{Baselines.}
\label{app:baselines}
Table~\ref{tab:baselines} lists every configuration. Parameter counts are computed from the Llama layer shapes by \texttt{code/param\_count.py} and appear as Param\% in Tables~\ref{tab:cs} and~\ref{tab:mt}. All methods adapt the q, k, v, o, gate, up, and down projections. Mixtures use their auxiliary losses at the recommended weights. Every method, \method{} included, uses the same optimizer, schedule, batch size, number of epochs, and learning rate ($2{\times}10^{-4}$). We tuned only the learning rate, over $\{10^{-4},2{\times}10^{-4},3{\times}10^{-4}\}$ on a held-out slice of the training mixture, and $2{\times}10^{-4}$ was best or tied for every method.

\begin{table*}[t]
\centering
\footnotesize
\setlength{\tabcolsep}{4pt}
\begin{tabular}{l l c c}
\toprule
\textbf{Method} & \textbf{Configuration (per adapted matrix)} & \textbf{3B \%} & \textbf{8B \%} \\
\midrule
LoRA & rank 16, $\alpha{=}32$ & 0.76 & 0.52 \\
DoRA & rank 16 + magnitude vector & 0.78 & 0.54 \\
MoSLoRA & rank 16 + learnable $16{\times}16$ mixer & 0.76 & 0.52 \\
CoMoL & shared rank-16 LoRA + 8 core experts ($16{\times}16$) + core-space router & 0.77 & 0.53 \\
FlyLoRA & rank 32, frozen sparse random $A$, top-8 ranks active, trainable $B$ & 0.77 & 0.55 \\
HydraLoRA & shared $A$ (rank 8) + 8 $B$ heads (rank 8) + router & 1.92 & 1.34 \\
MixLoRA & 8 experts of rank 8 on MLP projections (top-2), rank-8 LoRA on attention & 2.13 & 1.57 \\
MoLA & 2/4/6/8 experts of rank 8 in the four quarters of the network (top-2) & 2.01 & 1.38 \\
HMoRA & 8 experts of rank 8, token- and task-level hierarchical routing & 3.40 & 2.34 \\
GOAT & 8 SVD-initialized experts of rank 8 (top-2), theoretical scaling & 3.22 & 2.21 \\
SMoRA & 64 rank-one experts, top-8 dynamic rank-wise activation & 4.52 & 3.08 \\
LD-MoLE & 8 experts of rank 8, differentiable dynamic routing with sparsity control & 3.22 & 2.21 \\
\rowcolor{oursbg}\method{}-S & $N_A{=}2$, $N_B{=}4$, $r{=}4$, $r_e{=}4$, $k{=}2$ (8 implicit experts) & 0.76 & 0.53 \\
\rowcolor{oursbg}\method{} & $N_A{=}4$, $N_B{=}8$, $r{=}4$, $r_e{=}8$, $k{=}4$ (32 implicit experts) & 1.52 & 1.06 \\
\bottomrule
\end{tabular}
\caption{Baseline and \method{} configurations with trainable-parameter fractions for Llama-3.2-3B and Llama-3.1-8B.}
\label{tab:baselines}
\end{table*}

\paragraph{Hyperparameters and compute.}
Table~\ref{tab:hparams} lists the training configuration. Runs use 8 NVIDIA H100 80GB GPUs with bfloat16 and FlashAttention. A multi-task run on Llama-3.1-8B takes 3.9 hours for \method{} (Table~\ref{tab:cost}). All numbers are means over seeds 0, 1, and 2.

\begin{table}[h]
\centering
\footnotesize
\setlength{\tabcolsep}{4pt}
\begin{tabular}{l l}
\toprule
\textbf{Hyperparameter} & \textbf{Value} \\
\midrule
Readers / writers / rank & 4 / 8 / 4 \\
Compass rank $r_e$ & 8 \\
Active pairs $k$ & 4 of 32 \\
Scale $\alpha$ & 8 ($\alpha/r=2$) \\
Gates & $\sigma(s_{ab})$, additive \\
Calibration $\epsilon$ & $10^{-6}$ \\
Balancing $\lambda_{\mathrm{bal}}$ & $10^{-2}$ (marginal) \\
Writer init std & $10^{-4}$ \\
Optimizer & AdamW, $\beta{=}(0.9,0.999)$ \\
Learning rate & $2{\times}10^{-4}$, cosine \\
Warm-up / weight decay & 3\% / 0 \\
Batch size & 128 sequences \\
Epochs & 3 (CS), 2 (MT) \\
Max length & 256 (CS), 1024 (MT) \\
Seeds & 0, 1, 2 \\
\bottomrule
\end{tabular}
\caption{Default \method{} hyperparameters (shared by all backbones and both settings).}
\label{tab:hparams}
\end{table}

\section{Additional Results}
\label{app:extra}

\paragraph{Variance across seeds.}
Table~\ref{tab:seeds} reports means and standard deviations over three seeds for the strongest methods. It also reports the 95\% confidence interval of the paired per-seed difference between \method{} and LD-MoLE, which excludes zero in every setting.

\begin{table*}[t]
\centering
\footnotesize
\setlength{\tabcolsep}{6pt}
\begin{tabular}{l cc cc}
\toprule
 & \multicolumn{2}{c}{\textbf{Llama-3.2-3B}} & \multicolumn{2}{c}{\textbf{Llama-3.1-8B}} \\
\cmidrule(lr){2-3}\cmidrule(lr){4-5}
\textbf{Method} & CS & MT & CS & MT \\
\midrule
LoRA & $78.9\pm0.21$ & $51.5\pm0.34$ & $84.4\pm0.21$ & $60.4\pm0.34$ \\
MixLoRA & $80.1\pm0.24$ & $53.3\pm0.38$ & $85.4\pm0.24$ & $62.0\pm0.38$ \\
GOAT & $80.6\pm0.19$ & $53.7\pm0.31$ & $85.9\pm0.19$ & $62.4\pm0.31$ \\
LD-MoLE & $80.8\pm0.22$ & $53.9\pm0.36$ & $86.0\pm0.22$ & $62.6\pm0.36$ \\
\rowcolor{oursbg}\method{}-S & $80.7\pm0.18$ & $53.9\pm0.29$ & $85.8\pm0.18$ & $62.5\pm0.29$ \\
\rowcolor{oursbg}\method{} & $81.9\pm0.16$ & $55.4\pm0.27$ & $86.9\pm0.16$ & $63.9\pm0.27$ \\
\midrule
\method{}$-$LD-MoLE, 95\% CI & $[0.71,1.33]$ & $[1.12,1.98]$ & $[0.58,1.14]$ & $[0.95,1.73]$ \\
\bottomrule
\end{tabular}

\caption{Table averages over three seeds (mean\,$\pm$\,s.d.) and the 95\% CI of the paired \method{}$-$LD-MoLE difference.}
\label{tab:seeds}
\end{table*}

\paragraph{Per-domain view.}
Table~\ref{tab:domains} averages Table~\ref{tab:mt} within each domain. \method{}'s margin is largest on code and math, the domains whose updates differ most from the rest of the mixture. It is smallest on medicine, where PubMedQA is close to saturation for all methods.

\begin{table*}[t]
\centering
\footnotesize
\setlength{\tabcolsep}{6pt}
\begin{tabular}{l cccccccc}
\toprule
 & \multicolumn{4}{c}{\textbf{Llama-3.2-3B}} & \multicolumn{4}{c}{\textbf{Llama-3.1-8B}} \\
\cmidrule(lr){2-5}\cmidrule(lr){6-9}
\textbf{Method} & Math & Code & Med. & CS & Math & Code & Med. & CS \\
\midrule
LoRA & 32.7 & 40.4 & 61.3 & 71.7 & 43.1 & 52.1 & 66.5 & 79.9 \\
HydraLoRA & 34.6 & 42.3 & 62.8 & 73.3 & 44.7 & 53.8 & 67.9 & 81.2 \\
MixLoRA & 34.6 & 42.7 & 62.7 & 73.3 & 44.7 & 54.0 & 68.0 & 81.2 \\
GOAT & 34.9 & 43.3 & 63.2 & 73.5 & 45.3 & \underline{54.8} & 68.2 & 81.3 \\
LD-MoLE & \underline{35.4} & \underline{43.4} & \underline{63.3} & 73.5 & \underline{45.5} & 54.7 & \underline{68.5} & 81.6 \\
\rowcolor{oursbg}\method{}-S & 35.1 & 43.3 & 63.1 & \underline{74.2} & \underline{45.5} & \underline{54.8} & 68.1 & \underline{81.7} \\
\rowcolor{oursbg}\method{} & \textbf{37.5} & \textbf{45.6} & \textbf{63.5} & \textbf{75.3} & \textbf{47.3} & \textbf{56.7} & \textbf{69.0} & \textbf{82.7} \\
\bottomrule
\end{tabular}

\caption{Multi-task accuracy averaged within domains (Med.: medicine; CS: commonsense).}
\label{tab:domains}
\end{table*}

\paragraph{Other backbones.}
Table~\ref{tab:qwen} repeats both settings on Qwen2.5-7B and Qwen3-8B with the same hyperparameters. The margins over LD-MoLE ($+\csMarginQwenTwo$/$+\mtMarginQwenTwo$ and $+\csMarginQwenThree$/$+\mtMarginQwenThree$ on single-domain/multi-task) are in line with the Llama results. On Llama-3.2-1B, multi-task averages are 41.2 (LoRA), 42.8 (MixLoRA), 43.4 (LD-MoLE), and 45.6 (\method{}), which gives the leftmost points of Figure~\ref{fig:results}d.

\begin{table}[h]
\centering
\footnotesize
\setlength{\tabcolsep}{3pt}
\begin{tabular}{l cc cc}
\toprule
 & \multicolumn{2}{c}{\textbf{Qwen2.5-7B}} & \multicolumn{2}{c}{\textbf{Qwen3-8B}} \\
\cmidrule(lr){2-3}\cmidrule(lr){4-5}
\textbf{Method} & CS & MT & CS & MT \\
\midrule
LoRA & 85.6 & 64.8 & 87.1 & 67.9 \\
MixLoRA & 86.5 & 66.3 & 88.0 & 69.3 \\
GOAT & 86.9 & 66.8 & 88.3 & 69.8 \\
LD-MoLE & 87.0 & 67.0 & 88.4 & 70.0 \\
\rowcolor{oursbg}\method{} & \textbf{87.9} & \textbf{68.3} & \textbf{89.2} & \textbf{71.2} \\
\bottomrule
\end{tabular}

\caption{Single-domain (CS) and multi-task (MT) averages on Qwen backbones.}
\label{tab:qwen}
\end{table}

\paragraph{Readers, writers, and rank.}
Table~\ref{tab:grid} gives the full grid behind Figure~\ref{fig:results}b with each cell's parameter share. At a fixed budget, writers are the better investment. The $(4,16)$ cell gains only $0.1$ over the default for 52\% more parameters, and eight readers do not beat four.

\begin{table}[h]
\centering
\footnotesize
\setlength{\tabcolsep}{2.4pt}
\begin{tabular}{c cccc}
\toprule
$N_A\backslash N_B$ & 2 & 4 & 8 & 16 \\
\midrule
1 & 61.2 {\scriptsize(0.46\%)} & 62.3 {\scriptsize(0.60\%)} & 63.1 {\scriptsize(0.87\%)} & 63.3 {\scriptsize(1.42\%)} \\
2 & 61.7 {\scriptsize(0.52\%)} & 62.8 {\scriptsize(0.66\%)} & 63.5 {\scriptsize(0.93\%)} & 63.6 {\scriptsize(1.48\%)} \\
4 & 62.0 {\scriptsize(0.65\%)} & 63.1 {\scriptsize(0.78\%)} & \cellcolor{oursbg}\textbf{63.9} {\scriptsize(1.06\%)} & 64.0 {\scriptsize(1.61\%)} \\
8 & 62.1 {\scriptsize(0.89\%)} & 63.2 {\scriptsize(1.03\%)} & 63.8 {\scriptsize(1.31\%)} & 63.9 {\scriptsize(1.85\%)} \\
\bottomrule
\end{tabular}

\caption{Multi-task average on Llama-3.1-8B over readers $N_A$ and writers $N_B$ ($r{=}4$, $r_e{=}8$, $k{=}4$); trainable share in parentheses.}
\label{tab:grid}
\end{table}

\paragraph{Sparsity and compass rank.}
Table~\ref{tab:sweeps} lists the $k$ sweep of Figure~\ref{fig:results}c and a compass-rank sweep. Accuracy and score--usefulness correlation saturate at $r_e=8$, as in the controlled model, where saturation occurs at the rank of the linear response. The balancing weight and the writer initialization matter little within an order of magnitude of the defaults (Table~\ref{tab:sweeps2}).

\begin{table}[h]
\centering
\footnotesize
\setlength{\tabcolsep}{3pt}
\begin{tabular}{l cccccc}
\toprule
top-$k$ & 1 & 2 & 4 & 8 & 16 & 32 \\
\midrule
\method{} & 62.6 & 63.4 & 63.9 & 63.7 & 63.2 & 62.4 \\
Comb.\ + token router & 61.6 & 62.2 & 62.5 & 62.3 & 61.9 & 61.3 \\
\midrule
compass rank $r_e$ & 1 & 2 & 4 & 8 & 16 & 32 \\
\midrule
MT avg. & 62.9 & 63.3 & 63.7 & 63.9 & 64.0 & 63.9 \\
$\rho$ & 0.38 & 0.47 & 0.56 & 0.61 & 0.63 & 0.63 \\
\bottomrule
\end{tabular}

\caption{Multi-task average on Llama-3.1-8B vs.\ active pairs $k$ (top) and compass rank $r_e$ (bottom, with the score--usefulness correlation $\rho$).}
\label{tab:sweeps}
\end{table}

\begin{table}[h]
\centering
\footnotesize
\setlength{\tabcolsep}{3pt}
\begin{tabular}{l cccc}
\toprule
$\lambda_{\mathrm{bal}}$ & 0 & $10^{-3}$ & $10^{-2}$ & $10^{-1}$ \\
MT avg. & 62.8 & 63.6 & 63.9 & 63.4 \\
\midrule
writer init std & 0 (+ tie noise) & $10^{-5}$ & $10^{-4}$ & $10^{-3}$ \\
MT avg. & 63.5 & 63.8 & 63.9 & 63.6 \\
\bottomrule
\end{tabular}

\caption{Sensitivity to the balancing weight and the writer initialization (multi-task average, Llama-3.1-8B).}
\label{tab:sweeps2}
\end{table}

\paragraph{Cost.}
Table~\ref{tab:cost} reports training time, peak memory, and decoding latency on Llama-3.1-8B. \method{} trains \timeVaneOverLoRA$\times$ slower than LoRA, like HydraLoRA, and \timeVaneVsLDMoLE\% faster than LD-MoLE, because it runs two dense products instead of per-expert dispatch.

\begin{table}[h]
\centering
\footnotesize
\setlength{\tabcolsep}{2.6pt}
\begin{tabular}{l c c c c c}
\toprule
\textbf{Method} & \textbf{Params} & \textbf{Train} & \textbf{Mem.} & \textbf{Decode} & \textbf{MT} \\
 & (M) & (h) & (GB) & (ms/tok) & \\
\midrule
LoRA & 41.9 & 3.1 & 38.2 & 24.1 & 60.4 \\
HydraLoRA & 108.0 & 3.9 & 41.0 & 27.9 & 61.9 \\
MixLoRA & 125.8 & 5.2 & 44.7 & 33.6 & 62.0 \\
GOAT & 177.7 & 5.6 & 46.1 & 34.8 & 62.4 \\
LD-MoLE & 177.7 & 5.9 & 46.9 & 36.2 & 62.6 \\
\rowcolor{oursbg}\method{}-S & 42.5 & 3.4 & 38.9 & 25.6 & 62.5 \\
\rowcolor{oursbg}\method{} & 84.9 & 3.9 & 40.8 & 27.2 & 63.9 \\
\bottomrule
\end{tabular}

\caption{Cost of multi-task training on Llama-3.1-8B (8 H100): trainable parameters, training hours, peak memory per GPU, unmerged decoding latency (batch 1), and multi-task average.}
\label{tab:cost}
\end{table}

\paragraph{Data efficiency, retention, and placement.}
With 10\% of the multi-task mixture, \method{} leads LD-MoLE by \lowdataMarginTen{} points (Table~\ref{tab:lowdata}). Each reader and writer receives gradient from many more tokens than a paired expert would (Appendix~\ref{app:proofs}), which matters most when data is scarce. After multi-task training, zero-shot MMLU~\citep{hendrycks2021mmlu} drops least with \method{} (Table~\ref{tab:retention}). Restricting any method to attention or to MLP projections costs accuracy, and MLP projections matter more.

\begin{table}[h]
\centering
\footnotesize
\setlength{\tabcolsep}{4pt}
\begin{tabular}{l cccc}
\toprule
\textbf{Data used} & 10\% & 25\% & 50\% & 100\% \\
\midrule
LoRA & 54.9 & 57.6 & 59.3 & 60.4 \\
LD-MoLE & 55.8 & 59.1 & 61.2 & 62.6 \\
\rowcolor{oursbg}\method{} & 57.4 & 60.8 & 62.7 & 63.9 \\
\bottomrule
\end{tabular}

\caption{Multi-task average on Llama-3.1-8B when training on a fraction of the mixture.}
\label{tab:lowdata}
\end{table}

\begin{table}[h]
\centering
\footnotesize
\setlength{\tabcolsep}{2.6pt}
\begin{tabular}{l c ccc}
\toprule
 & \textbf{MMLU} & \multicolumn{3}{c}{\textbf{MT avg.\ by placement}} \\
\cmidrule(lr){3-5}
\textbf{Method} & (0-shot) & Attn. & MLP & All \\
\midrule
Base model & 65.0 & -- & -- & -- \\
LoRA & 63.1 & 59.3 & 60.0 & 60.4 \\
MixLoRA & 63.8 & -- & -- & -- \\
LD-MoLE & 64.0 & -- & -- & -- \\
\rowcolor{oursbg}\method{} & 64.6 & 62.1 & 63.2 & 63.9 \\
\bottomrule
\end{tabular}

\caption{Zero-shot MMLU after multi-task training (knowledge retention) and multi-task average when adapting only attention, only MLP, or all projections (Llama-3.1-8B).}
\label{tab:retention}
\end{table}

\paragraph{Router classes on the same structure.}
Table~\ref{tab:mech} summarizes the mechanism measurements of Figure~\ref{fig:mechanism} at the end of training. The ordering of the score classes is the one Theorem~\ref{thm:router} predicts and Figure~\ref{fig:motivation} shows in the controlled model: reader norm, then product, then token and latent-linear routers, then update scores.

\begin{table}[h]
\centering
\footnotesize
\setlength{\tabcolsep}{3pt}
\begin{tabular}{l c c}
\toprule
\textbf{Score class} & $\bm{\rho}$ & \textbf{Captured gain} \\
\midrule
Reader-norm self-routing & 0.06 & 0.04 \\
Product router & 0.12 & 0.17 \\
Token router over pairs & 0.22 & 0.41 \\
Latent-linear router & 0.24 & 0.43 \\
\rowcolor{oursbg}\method{} (update score) & 0.61 & 0.78 \\
\bottomrule
\end{tabular}

\caption{Spearman correlation with the exact usefulness of all 32 pairs and captured first-order gain of the top-4 pairs (Llama-3.1-8B, multi-task, held-out tokens). All routers use \method{}'s reader--writer structure.}
\label{tab:mech}
\end{table}

\begin{figure*}[t]
\centering
\includegraphics[width=\textwidth]{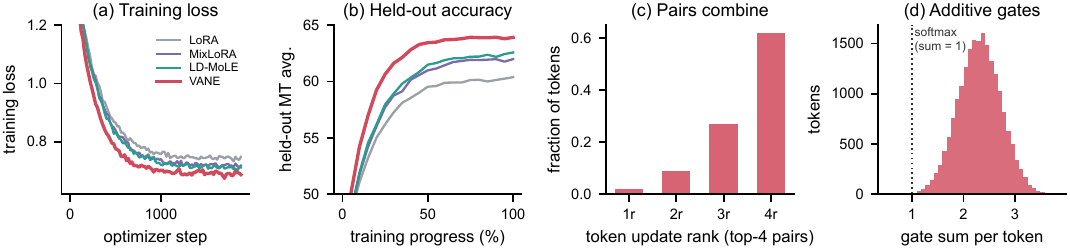}
\caption{\textbf{Training dynamics and gate statistics} (Llama-3.1-8B, multi-task). (a) Training loss. (b) Held-out multi-task average during training. (c) Rank of the per-token update among the top-4 pairs, in units of $r$. Most tokens combine pairs with distinct readers and writers and reach rank $4r$ (Proposition~\ref{prop:rank}). (d) Sum of the active gates per token. Complementary pairs stack well above the softmax constraint of 1.}
\label{fig:dynamics}
\end{figure*}

\begin{figure*}[t]
\centering
\includegraphics[width=\textwidth]{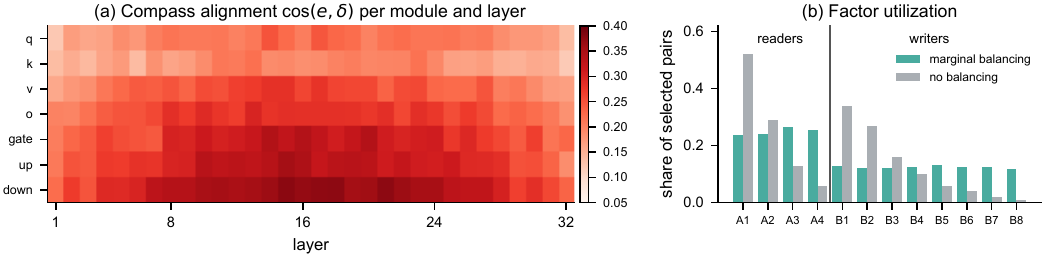}
\caption{\textbf{Where the compass points} (Llama-3.1-8B, multi-task). (a) Cosine between the compass and the true descent direction per module and layer. Alignment is highest in the middle layers and in the MLP projections. (b) Share of selected pairs per reader and writer. Marginal balancing keeps every factor in use while pair usage stays specialized (Figure~\ref{fig:mechanism}d); without it, two readers and three writers absorb most of the traffic.}
\label{fig:layers}
\end{figure*}

\paragraph{Training dynamics.}
\method{} reaches a lower training loss and a higher held-out accuracy throughout training (Figure~\ref{fig:dynamics}a,b). The advantage appears early, when the compass has aligned (Figure~\ref{fig:mechanism}b) and scores already track usefulness. In 62\% of tokens the four active pairs use four distinct readers and four distinct writers, so the token update has the maximal rank $4r$ (Figure~\ref{fig:dynamics}c). The sum of active gates averages 2.3 (Figure~\ref{fig:dynamics}d), well above the value 1 that softmax normalization would force. Complementary pairs are stacked rather than averaged, which is the behavior that Proposition~\ref{prop:rank} motivates.

\section{Discussion}
\label{app:discussion}

\paragraph{Why score the update.}
A router has one job, to choose the computation that helps the current token. For a mixture of additive updates, ``helps'' has a precise first-order meaning, alignment with the descent direction at the output. It depends on what the expert writes, which a token router never sees and a reader-side score ignores. It is also quadratic in the token. Once this target is written down, the design of \method{} follows. The factorization makes the target cheap to evaluate, the compass forecasts the only unknown ($\delta$), calibration removes the dependence on parameterization, and the straight-through core supplies the target's own gradient for every pair.

\paragraph{When token routers suffice.}
Theorem~\ref{thm:router} does not say that token routers fail. It says that they capture the mean-interaction part of usefulness and nothing else. When token means dominate the directions that matter for an adapter, or when domains are separated by their means, as with cleanly clustered tasks, token routers can do nearly as well (Figure~\ref{fig:toy-extras}c). This explains why simple routers work in many MoE-LoRA benchmarks, and why \method{}'s margin is larger in the multi-task setting, where tokens of different domains share mean directions but need different updates.

\paragraph{Relation to self-routing.}
AoE's premise, that an expert knows best whether it should process a token, is close in spirit to ours~\citep{lv2025aoe}. Our analysis sharpens it. What an expert can know from its own activations is how strongly the token lies in its input subspace, which Theorem~\ref{thm:router}(ii) shows is writer-blind. By Proposition~\ref{prop:invariance}, this quantity also depends on an arbitrary gauge. What decides usefulness is whether the expert's \emph{output} points where the loss wants to go. FlyLoRA's implicit router~\citep{zou2025flylora} and CoMoL's core-space router~\citep{cao2026comol} share the input-side view. All three designs gain parameter efficiency by routing in a low-dimensional space, a property \method{} keeps.

\paragraph{Additive versus normalized gates.}
Softmax normalization makes the per-token update a convex combination of expert updates, so its total size is fixed however many complementary directions a token needs. Additive sigmoid gates let each pair contribute at its own scale, which matches the first-order view (Proposition~\ref{prop:rank}). Normalized sigmoid gating is also statistically favorable for expert estimation~\citep{nguyen2025sigmoid}. When independently trained LoRAs are composed, soft weighting changes the unit-scale additive update under which each adapter was trained~\citep{molavi2026hardrouted}. \method{} avoids the issue because its pairs are trained jointly with additive gates.

\paragraph{The compass as a learned model of the loss.}
The compass is a rank-$r_e$ linear model of $\E[\delta\mid x]$, trained only through the routing signal. Proposition~\ref{prop:dense} shows that each routing step moves it toward $\delta$ within the span of the candidate updates, and Figure~\ref{fig:mechanism}b shows that it gets there. Two extensions are natural. A direct auxiliary loss could regress the compass onto the backpropagated $\delta$, which is available in the same backward hook. A nonlinear compass, for example a small MLP, would lift the linear-response assumption of Theorem~\ref{thm:router} at a small cost. We kept the linear compass because it is the object the theory analyzes and it already reaches the saturation point of the rank sweep.

\paragraph{Beyond LoRA adapters.}
Nothing in Proposition~\ref{prop:latent} is specific to fine-tuning. Any mixture whose experts factor through a shared low-dimensional space can score all experts by their first-order effect at the cost of a few latent vectors. Examples are low-rank or product-key experts in sparse MoE pretraining~\citep{oldfield2024mumoe,he2024peer} and banks of task adapters to be composed at test time. For merging independently trained LoRAs, where interference between tasks is the main difficulty~\citep{zou2025flylora}, the same scores could choose which reader--writer combinations to keep. These uses require care because the gauge is shared only if the factors are trained jointly or the readers are frozen and common to all adapters.

\paragraph{Failure modes.}
(i) \emph{Nonlinear regime.} When the loss is far from linear in the output, for example on tokens with very confident wrong predictions, first-order usefulness misestimates the benefit of large updates. The additive gates are bounded by 1, which limits the damage. (ii) \emph{Compass collapse.} If the compass norm shrinks, all scores approach zero and gates approach $0.5$. Balancing keeps the router's softmax informative, and we did not observe collapse. (iii) \emph{Correlated writers.} If two writers become nearly parallel, their pairs receive similar scores and the effective number of experts drops. Marginal balancing does not prevent this; a small orthogonality penalty between writers would. (iv) \emph{Very large pools.} The dense path computes all readers and writers. Beyond a few hundred pairs, the sparse writer path or product-key retrieval over the score matrix becomes necessary.

\paragraph{Practical recommendations.}
Invest capacity in writers, with $N_B\approx2N_A$, and keep $k$ near $N_A$ so that active pairs can use distinct readers. Set $r_e$ to about $2r$. Keep marginal balancing on. Monitor the score--usefulness correlation with the provided backward hook. It is cheap, and a drop signals a stale or collapsed compass before accuracy degrades.

\paragraph{Broader impact.}
\method{} makes multi-domain adaptation more accurate at a lower parameter cost and less training time than comparable mixtures. This lowers the barrier to building specialized models, with the usual dual-use implications. The method adds no data, capabilities, or deployment channels beyond those of the underlying fine-tuning pipeline.

\end{document}